\RequirePackage{fix-cm}

\documentclass[twocolumn]{svjour3}          
\journalname{Journal of Mathematical Imaging and Vision}

\smartqed  
\usepackage{amsfonts}
\usepackage{color}
\usepackage{graphicx}
\usepackage{amssymb}
\usepackage{amsmath}
\usepackage{array}
\usepackage{tabularx}
\usepackage{multirow}
\usepackage{longtable}
\usepackage[hyphens]{url}
\usepackage{mathtools}
\usepackage[utf8]{inputenc}
\usepackage[english]{babel}
\mathtoolsset{showonlyrefs}
\usepackage{enumitem}
\usepackage[sort, numbers]{natbib}

\usepackage[colorinlistoftodos]{todonotes}

\usepackage[bookmarks,colorlinks,breaklinks]{hyperref}  
\usepackage[hyphenbreaks]{breakurl}

\definecolor{dullmagenta}{rgb}{0.4,0,0.4}   
\definecolor{darkblue}{rgb}{0,0,0.4}
\definecolor{darkgreen}{rgb}{0,0.6,0}
\definecolor{darkred}{rgb}{0.6,0,0}

\hypersetup{linkcolor=darkblue,citecolor=blue,filecolor=dullmagenta,urlcolor=darkblue,breaklinks=true}

\DeclareMathOperator{\Circ}{Circ}
\DeclareMathOperator{\diag}{diag}
\DeclareMathOperator{\avg}{avg}

\DeclareMathOperator{\PS}{PS}
\DeclareMathOperator{\BS}{BS}
\DeclareMathOperator{\RPS}{RPS}
\DeclareMathOperator{\RBS}{RBS}

\newcommand{\bN}{{\mathbb N}} 
\newcommand{\bZ}{{\mathbb Z}} 
\newcommand{\bR}{{\mathbb R}} 
\newcommand{\bC}{{\mathbb C}} 
\newcommand{\bG}{{\mathbb G}}

\newcommand{\idty}{{\rm I}} 

\newcommand{\cA}{{\cal A}}
\newcommand{\cB}{{\cal B}} 

\newcommand{\cF}{{\cal F}} 
\newcommand{\cG}{{\cal G}}

\newcommand{\cL}{{\cal L}}

\newcommand{\cR}{{\cal R}}
\newcommand{\cS}{{\cal S}} 

\newcommand{\cX}{{\cal X}}

\newcommand{\bem}{\left(\!\begin{array}}
\newcommand{\eem}{\end{array}\!\right)}
\newcommand{\bsm}{\left(\begin{smallmatrix}} 
\newcommand{\esm}{\end{smallmatrix}\right)}  

\newcommand{\supp}{{\rm supp}}

\title{Fourier descriptors based on the structure of the human primary visual cortex with applications to object recognition}
\author{Amine Bohi \and Dario Prandi \and Vincente Guis \and Fr\'ed\'eric Bouchara \and Jean-Paul Gauthier }

\begin{document}

\institute{A. Bohi \and V. Guis \and F. Bouchara \and J.P. Gauthier \at
              LSIS Laboratory, University of South Toulon - Var B.P. 20132 83957 La Garde Cedex \\
              Tel.: +33-4-94142210\\
               \email{amine.bohi@lsis.org}; $\{$bouchara,gauthier,guis$\}$@univ-tln.fr         \\
           \and
           D. Prandi \at
            CEREMADE Laboratory, University of Paris-Dauphine, Place du Maréchal De Lattre De Tassigny 75775 PARIS CEDEX 16 \\
            \email{prandi@ceremade.dauphine.fr}
}

\date{Received: date / Accepted: date}

\maketitle

\begin{abstract}
In this paper we propose a supervised object recognition method using new global features and inspired by the model of the human primary visual cortex V1 as the semidiscrete roto-translation group $SE(2,N) = \bZ_N\rtimes \bR^2$. The proposed technique is based on generalized Fourier descriptors on the latter group, which are invariant to natural geometric transformations (rotations, translations). 
These descriptors are then used to feed an SVM classifier. 
We have tested our method against the COIL-100 image database and the ORL face database, and compared it with other techniques based on traditional descriptors, global and local. The obtained results have shown that our approach looks extremely efficient and stable to noise, in presence of which it outperforms the other techniques analyzed in the paper.

\end{abstract}

\section{Introduction}
Object recognition is a fundamental problem in computer vision and keeps attracting more and more attention nowadays. Its concepts have been applied in multiple fields, as manufacturing, surveillance system, optical character recognition, face recognition, etc.

Almost every object recognition algorithm proposed in the literature is based on the computation of certain features of the image, which allow to characterize the object depicted and to discriminate it from others. In particular, since objects can appear at different locations and with different sizes, it is desirable for such features to be invariant by translation, rotation and scale. These invariant features can be global, i.e.\ computed taking into account the whole image, or local, i.e.\ computed considering only neighborhoods of key-points in the image.

In this paper we focus on \emph{Fourier descriptors}, an important class of global invariant features used since the seventies \citep{Granlund1972, Zahn1972Fourier} based on algebraic properties of the Fourier transform. In particular, inspired by some neurophysiological facts on the structure of the human primary visual cortex, we extend this theory to define features invariant to translation and rotation and we apply them for invariant object recognition in SVM context. These results are then compared with those obtained with another important class of global invariant features, the \emph{moment invariants} (see Appendix~\ref{sec:moment}), used e.g.\ in \citep{399871, raja2011artificial, rajasekaran2000image}, and with two different local invariants. For more information on object recognition via local features we refer to \cite{lowe2004distinctive,morel2009asift,ke2004pca,mikolajczyk2005performance,bay2006surf,dalal2005histograms}. 

Our choice of a global approach is motivated by the better results obtained by these methods in presence of noise, luminance changes and other different alterations, with respect to algorithms based on local invariant features \citep{Choksuriwong2008}. Indeed, under these conditions, key-points detectors used in the local approach produce key-points that are not relevant for the object recognition. 

In the following we briefly introduce the theory of Fourier descriptors, before discussing the framework used in this paper and our contributions.



%
%

\subsection{Fourier descriptors}

The basic idea behind Fourier descriptors is that the action of an abelian locally compact group $\bG$ on functions in $L^2(\bG)$ is much easier to treat at the level of their Fourier transforms.
In the specific case of images, $f,g\in L^2(\bR^2)$, this is expressed by the well-known equivalence for the translation of $a\in\bR^2$:
\begin{multline}
	\label{eq:translations}
	f(x) = g(x-a) \quad \forall x\in\bR^2
	\\
	\iff
	\hat f(\lambda) = e^{i\langle a, \lambda \rangle} \hat g(\lambda) \quad\forall \lambda\in \bR^2,
\end{multline}
where the Fourier transform is defined\footnote{Here we use a non-unitary definition of the Fourier transform for future convenience in computations.} by
\begin{equation}
	\hat f(\lambda) = \int_{\bR^2} f(x)\,e^{-i\langle\lambda, x\rangle}\,{dx}, \qquad\forall \lambda\in\bR^2.
\end{equation}

In this setting, Fourier descriptors are quantities associated with functions of $L^2(\bR^2)$ that can be easily computed starting from their Fourier representation and that are invariant under the action of translations.
Ideally, a Fourier descriptor should be \textbf{complete}, meaning that for any couple of images $f,g\in L^2(\bR^2)$ the equality of the Fourier descriptor is equivalent to the equality of $f$ and $g$ up to translations.
Indeed, the lack of completeness could yield to problems in applications, notably to false positives in the classification.

However, a result as strong as completeness is usually out of reach and unnecessary for practical applications. 
In this case, one looks for Fourier descriptors that are at least  \textbf{weakly complete}, meaning that they are complete on a sufficiently big subset of $L^2(\bR^2)$, usually either open and dense or at least residual, i.e.\ the intersection of countably many open and dense sets. 
This guarantees that the Fourier descriptor will correctly classify a sufficiently large class of images. 

Various Fourier descriptors have been defined in the literature \cite{Granlund1972, Zahn1972Fourier, Kuhl1982, Kakarala2012Bispectrum, Smach2008Generalized}.
In this work we are mainly interested in the following two, whose invariance w.r.t.\ translations can be checked via \eqref{eq:translations}.
\begin{itemize}[leftmargin=*]
	\item \textbf{Power-spectrum:} the quantity $\PS_f(\lambda) := |\hat f(\lambda)|^2$ for $\lambda\in\bR^2$, which is the Fourier transform of the auto-correlation function 
	\begin{equation}
		a_f(x) := \int_{\bR^2} \overline{f(y)} \, f(y+x)\, dy.
	\end{equation}
	It is easy to show that the power-spectrum is not weakly complete, and indeed it is used in texture synthesis to identify the translation invariant Gaussian distribution of textures \citep{Galerne2011Random}.

	\item \textbf{Bispectrum:} an extension of the power-spectrum, it is the quantity $\BS_f(\lambda_1,\lambda_2):=\hat f(\lambda_1)\, \hat f(\lambda_2)\, \overline{\hat f(\lambda_1+\lambda_2)}$, or equivalently the Fourier transform of the triple correlation,  defined as
	\begin{equation}
		a_{3,f}(x_1,x_2) := \int_{\bR^2} \overline{f(y)} \, f(y+x_1)\, f(y+x_2) \,dy.
	\end{equation}
	These descriptors are complete on compactly supported functions of $L^2(\bR^2)$ and are well established in statistical signal processing.
	See e.g.\ \citep{Dubnov1997}, where they are applied to sound texture recognition.
\end{itemize}

These two Fourier descriptors can be easily generalized to functions on $L^2(\bG)$ of a locally compact abelian group $\bG$ to obtain invariants under the action of $\bG$.
This can be applied, for example, to 2D shape recognition.
However, when working with images, these descriptors are unsatisfying. Indeed, they are invariant only under translations, and so cannot be used to classify images under the action of rotations.

\subsection{Framework of the paper}

In this paper, following a line of research started in \citep{Smach2008Generalized}, we present a theoretical framework that allows us to build generalized Fourier descriptors which are invariant w.r.t.\ (semidiscrete) roto-translations of images.
We exploit the following two facts:
\begin{itemize}
	\item It is possible to define a natural generalization of the power spectrum and the bispectrum on non-commutative groups, as it has been done in \citep{Smach2008Generalized,Kakarala2012Bispectrum}. 
	\item Contributions of some of the authors to a fairly recent model of the human primary visual cortex V1 \citep{Boscain2014Hypoelliptic,Boscain2014Image} have shown that the latter can be modeled as the semidiscrete group of roto-translations $SE(2,N) = \bZ_N\rtimes \bR^2$.
	In this model, cortical stimuli are functions in $L^2(SE(2,N))$, w.r.t.\ the Haar measure of $SE(2,N)$, and images from the visual plane are lifted to cortical stimuli via a natural injective and left-invariant lift operation $\cL:L^2(\bR^2)\to L^2(SE(2,N))$.
	Such lift is defined as the wavelet transform w.r.t.\ a mother wavelet $\Psi$, see Section~\ref{sec:a_mathematical_model_of_the_primary_visual_cortex_v1}.
\end{itemize}

From these facts, a natural pipeline for invariant object recognition is the following:
\begin{enumerate}
	\item Given an image $f\in L^2(\bR^2)$ lift it to a cortical stimulus $\cL f\in L^2(SE(2,N))$.
	\item Compute the generalized Fourier descriptors of $\cL f$ on the non-commutative group $SE(2,N)$.
	\item If the lift of another image $g\in L^2(\bR^2)$ have the same Fourier descriptors as $\cL f$, 
		deduce that $\cL f \approx \cL g$ up to the action of $SE(2,N)$.
	\item Thanks to the left-invariance and injectivity of the lift $\cL$, obtain that also $f\approx g$ up to the action of $SE(2,N)$.
\end{enumerate}

This pipeline was already investigated in \cite{Smach2008Generalized}, where the authors considered a non-left-invariant lift, the \emph{cyclic lift}. For this lift they then proved a weak completeness result of the generalized bispectrum for images, represented as $L^2(\mathbb R^2)$ functions with support inside a fixed compact set. 

In this paper we consider the same question for left-invariant lifts, where the situation turns out to be more complicated. In particular, as explained in the following section, to ensure the weak completeness we are led to consider ``stronger'' invariants than the generalized bispectrum. However, as observed in Remark~\ref{rmk:cyclic-comp}, the actual computation of these stronger invariants on lifted images requires $N$ times less computational time and space w.r.t.\ the computation of the generalized bispectrum of cyclically lifted images.

\subsection{Contributions of the paper}

Let $K\subset \bR^2$ be any compact set, representing the size of the images under consideration. 
According to the pipeline for object recognition introduced above, the weak completeness of the generalized Fourier descriptors on images can be proved in two steps:
\begin{enumerate}
	\item Prove the completeness of the generalized Fourier descriptors on some residual set $\mathcal G\subset L^2(\bZ_N\times K)$ of cortical stimuli;
	\item Prove that for some residual set $\mathcal R\subset L^2(\mathbb R^2)$ of images with support in $K$ we have $\mathcal L(\mathcal R)\subset \mathcal G$.
\end{enumerate}

The first point is addressed in Theorem~\ref{thm:complete-bisp}, where is identified an open and dense set $\cG\subset L^2(\bZ_N\times K)$ on which the combination of the generalized power-spectrum and bispectrum holds. This generalizes the result in \cite{Smach2008Generalized}, where the same result was proved for a residual subset of the range of the cyclic lift.

Unfortunately, it turns out that for this set $\cG$ and a left-invariant lift $\cL$ there is no hope of finding a set $\cR\subset L^2(\mathbb R^2)$ satisfying the second point above. We are then led to introduce stronger Fourier descriptors, the rotational power-spectrum and bispectrum, which are invariant only w.r.t.\ rotations. To solve this problem we preprocess images by centering them at their barycenter, a procedure that is essential also in \citep{Smach2008Generalized}. Theorem~\ref{thm:rot-bisp-compl} then shows that the resulting invariants are complete for an open and dense set of functions in $L^2(K)$, for any compact $K\subset \bR^2$. The proof of this completeness requires fine technical tools from harmonic analysis and the theory of circulant operators, and for this reason we  only present a sketch of it, evidencing the technical difficulties. A complete proof will be presented in a forthcoming paper by the second and last authors.

Finally, in Theorem~\ref{thm:explicit-comp} we show that, under mild assumptions on the mother wavelet $\Psi$, to check the equality of all these Fourier descriptors it is sufficient to compute simple quantities computed  from the 2D Fourier transform of the image. This allow for an efficient implementation on regular hexagonal grids. After using these descriptors to feed a SVM based classifier, we compare  their performances with those of Hu and Zernike moments, the Fourier-Mellin transform and some well-known local descriptors. To this purpose, we test them on two large databases: the COIL-100\footnote{\url{http://www.cs.columbia.edu/CAVE/software/softlib/coil-100.php}} object recognition database, composed of 7200 objects presenting rotation and scale changes \cite{nene1996columbia}, and the ORL\footnote{\url{http://www.cl.cam.ac.uk/research/dtg/attarchive/facedatabase.html}} face database, on which different human faces are subjects to several kind of variations.

\subsection{Structure of the paper}

The remainder of the paper is organized as follows. In Section~\ref{sec:a_mathematical_model_of_the_primary_visual_cortex_v1}, we present the features of a mathematical model of the primary visual cortex $V1$ that are essential to our approach. In Section~\ref{sec:preliminaries_on_non_commutative_harmonic_analysis}, we introduce some generalities on the Fourier Transform on the semidiscrete group of roto-translation $SE(2,N)$. In Section~\ref{sec:fourier_descriptor_on_se}, we describe the natural generalization of the power-spectrum and the bispectrum on $\bR^2$ to $SE(2,N)$. We then prove the weak completeness result (Theorem~\ref{thm:complete-bisp}) and show that under the chosen lift operator this does not imply weak completeness for images. Finally, we introduce the rotational power-spectrum and bispectrum and sketch the proof of the corresponding weak completeness result (Theorem~\ref{thm:rot-bisp-compl}) for images. We end this section with some result on the practical computation of these descriptors. In Section~\ref{sec:experimental_results} we illustrate some numerical results where these descriptors are compared with those obtained via global descriptors such as Zernike moments, Hu moments, Fourier-Mellin transform, and local ones like the SIFT and HoG descriptors. Finally, we conclude with some practical suggestions in Section \ref{sec:conclusion}.



\section{A mathematical model of the primary visual cortex V1} 
\label{sec:a_mathematical_model_of_the_primary_visual_cortex_v1}

As mentioned in the introduction, the main novelty of our approach is its connection with a fairly recent model of the human primary visual cortex V1 due to Petitot and Citti-Sarti \citep{Petitot2008Neurogeometrie,Citti2006Cortical} and our recent contributions \citep{Boscain2012Anthropomorphic,Boscain2014Hypoelliptic,Boscain2014Image,Prandi2015}.
The theory of orientation scores introduced in \cite{Duits2010LeftInvarianta,Duits2010LeftInvariant} is also strongly connected with this work, in particular for its exploitation of left-invariant lift operators.
We also mention \cite{Mallat}, where image invariants based on the structure of the roto-translation group $SE(2)$ are introduced for textures.
In this section we present the features of this model that are essential to our approach.

Since it is well-known \citep{Hubel1959Receptive} that neurons in V1 are sensitive not only to positions in the visual field but also to local orientations and that it is reasonable to assume these orientations to be finite, in \citep{Boscain2014Hypoelliptic} V1 has been modeled as the semidiscrete group of roto-translations $SE(2,N) = \bZ_N\rtimes \bR^2$ for some even $N\in\bN$.
Letting $R_k$ be the rotation of $2\pi/k$, the (non-commutative) group operation of $SE(2,N)$ is
\begin{equation}
	(x,k)(y,r) = (x+R_k y, {k+r}).
\end{equation}
Here, we are implicitly identifying $k+r$ with $k+r\mod N$.

Visual stimuli $f\in L^2(\bR^2)$ are assumed to be lifted to activation patterns in $L^2(SE(2,N))$ by a lift operator $\cL:L^2(\bR^2)\to L^2(SE(2,N))$.
Motivated by neurophysiological evidence, we then assume that
\begin{enumerate}
	\item[(H)] The lift operator $\cL$ is linear and is defined as 
	\begin{equation}
		\label{eq:twist-shift}
		\cL f(x,k) := \int_{\bR^2} f(y) \bar\Psi(R_{-k}(y-x))\, dy,
	\end{equation}
	for a given mother wavelet $\Psi\in L^2(\bR^2)$ such that $\cL$ is injective and bounded.
\end{enumerate}

\begin{remark}
This assumption means that the lift operator under consideration is the wavelet transform w.r.t.\ $\Psi$ (See, e.g., \cite{Fuhr2002a}).
The fact that $\cL$ be injective and bounded is then equivalent to the fact that the mother wavelet $\Psi$ is \emph{weakly admissible}, i.e., is such that the map
\begin{equation}
	\lambda\in\mathbb R^2\mapsto \sum_{k\in \bZ_N} |\hat\Psi(R_{-k}\lambda)|^2
\end{equation}
is strictly positive and essentially bounded. 
\end{remark}

As a consequence of the above assumption, the lift operation $\cL$ is left-invariant w.r.t. to the action of \\$SE(2,N)$. 
Namely,
\begin{equation}
	\label{eq:left-inv}
	\Lambda(x,k)\circ \cL = \cL \circ\pi(x,k).
\end{equation}
Here $\Lambda$ and $\pi$ are the actions of $SE(2,N)$ on \\$L^2(SE(2,N))$ and $L^2(\bR^2)$ respectively.
That is,
\begin{gather}
	\begin{split}
	[\Lambda(x,k)&\varphi ](y,r) \\
	&= \varphi\left((x,k)^{-1}(y,r)\right) = \varphi(R_{-k}( y-x), {k+r}),
	\end{split}
	\\
	[\pi(x,k)f](y) = f\left((x,k)^{-1}y\right) = f(R_{-k}( y-x)) .
\end{gather}
Formula \eqref{eq:left-inv} can be seen as a semidiscrete version of the shift-twist symmetry \citep{Bressloff2001Geometric}. 

The main observation for our purposes is that \eqref{eq:left-inv} means that two images $f$ and $g\in L^2(\bR^2)$ can be deduced via roto-translation (i.e., $f = \pi(x,k)g$ for some $(x,k)\in SE(2,N)$) if and only if their lifts can be deduced via $\Lambda(x,k)$.


\section{Preliminaries on non-commutative harmonic analysis} 
\label{sec:preliminaries_on_non_commutative_harmonic_analysis}

In this section we introduce some generalities on the (non-commutative) Fourier transform on $SE(2,N)$, an essential tool to define and compute the Fourier descriptors we are interested in. We refer to \cite{Hewitt1963Abstract, Barut77theoryof} for a general introduction to the topic.

Since $SE(2,N)$ is a non-commutative unimodular group, the Fourier transform of $\varphi\in L^2(SE(2,N))$ is an operator associating to each (continuous) irreducible unitary representations $T^\lambda$ of $SE(2,N)$ some Hilbert-Schmidt operator on the Hilbert space where $T^\lambda$ acts.
Here, $\lambda$ is an index taking values in the dual object of $SE(2,N)$, which is denoted by $\widehat{SE(2,N)}$ and is the set of equivalence classes of irreducible unitary representations. 

The set of irreducible representations of a semi-direct product group can be obtained via Mackey's machinery (see, e.g., \cite[Ch. 17.1, Theorems 4 and 5]{Barut77theoryof}). Accordingly, $\widehat{SE(2,N)}$ is parametrized by the orbits of the (contragredient) action of rotations $\{R_k\}_{k\in\bZ_N}$ on $\mathbb R^2$, i.e., by the slice $\cS\subset\bR^2$ which in polar coordinates is $(0,+\infty)\times [0,2\pi/N)$. Additionally, corresponding to the origin, there are the characters of $\bZ_N$.
Namely, to each $\lambda\in\cS$ corresponds the representation $T^\lambda$ acting on $\bC^N$ via
\begin{multline}
	\label{eq:repre}
	T^\lambda(x,k)v 
	= \text{diag}_h(e^{i\langle \lambda, R_h x \rangle})\circ S^k v \\
	= \left( e^{i\langle \lambda, R_h x\rangle}v_{h+k} \right)_{h=0}^{N-1},
\end{multline}
where we denoted by $\diag_h v_h$ the diagonal matrix of diagonal $v\in\bC^N$ and by $S$ the shift operator $(Sv)_j = v_{j+1}$, so that $(S^kv)_j = v_{j+k}$.
On the other hand, to each $k\in\bZ_N$ corresponds the representation on $\bC$ given by $z\mapsto e^{i \frac{2\pi k}N} z$.
Since it is possible to show that to invert the Fourier transform it is enough to consider only the representations parametrized by $\cS$, we will henceforth ignore the $\bZ_N$ part of the dual.

Finally, the matrix-valued Fourier coefficient of a function $\varphi\in L^2(SE(2,N))\cap L^1(SE(2,N))$ for $\lambda\in\widehat{SE(2,N)}$ is
\begin{equation}
	\label{eq:ft-def}
	\hat\varphi(T^\lambda) = \int_{SE(2,N)} \varphi(a)\,T^\lambda(a^{-1})\,da,
\end{equation}
where $da$ is the Haar measure\footnote{That is, up to a multiplicative constant, the only left and right invariant measure on $SE(2,N)$. One can check that $\int_{SE(2,N)} \varphi(a)\,da = \sum_{k=0}^{N-1}\int_{\bR^2}\varphi(x,k)\,dx$.} of $SE(2,N)$.
This is essentially the same formula for the Fourier transform on $\bR^2$, which is a scalar and is obtained using the representations $\lambda(x) = e^{i \langle x, \lambda\rangle}$ acting on $\bC$.

Straightforward computations yield
\begin{equation}
	\label{eq:non-comm-ft}
    \hat \varphi(T^\lambda)_{i,j} = \cF(\varphi(\cdot,i-j))(R_{-j}\lambda),
\end{equation}
where we let $\cF$ denote the Fourier transform on $\bR^2$.  

As usual, the definition of the Fourier transform can be extended to the whole $L^2(SE(2,N))$ by density arguments.
Then, there exists a unique measure on $\widehat{SE(2,N)}$, the Plancherel measure, supported on $\cS$ where it coincides with the restriction of the Lebesgue measure of $\bR^2$, such that the Fourier transform is an isometry between $L^2(SE(2,N))$ and $L^2(\widehat{SE(2,N)})$.
In particular, the following inversion formula holds
\begin{equation}
	\varphi(x,k) = \int_{\cS}\text{Tr}\left(\hat\varphi(T^\lambda)\circ T^\lambda(x,k)\right)\,d\lambda.
\end{equation}

The fundamental property of the non-commutative Fourier transform, generalizing \eqref{eq:translations}, is that for all $\varphi,\eta\in L^2(SE(2,N))$ and $a\in SE(2,N)$ it holds
\begin{multline}
	\label{eq:fundamental}
	\varphi(x,k)=[\Lambda(a)\eta](x,k) \quad\forall (x,k)\in SE(2,N) \iff \\
	 \hat \varphi(T) = \hat\eta(T)\circ T^{-1}(a) \quad\forall T\in \widehat{SE(2,N)}.
\end{multline}
Namely, $\varphi,\eta$ can be deduced via the action of $SE(2,N)$ if and only if their Fourier transforms at a representation $T$ can be deduced via multiplication by $T(a)$.

\begin{remark}
The fact that the Fourier transform in \eqref{eq:ft-def} be matrix-valued is a direct consequence of $SE(2,N)$ being a Moore group, that is, that all the $T^\lambda$ act on finite-dimensional spaces.
This is not true for the roto-translation group $SE(2)$. As a consequence, the Fourier transform on $SE(2)$, takes values not in the finite dimensional space of complex $N\times N$ matrices, but in the infinite dimensional space of operators over $L^2(\mathbb S^1)$.
This is indeed the main theoretical advantage of considering $SE(2,N)$.	
\end{remark}

\subsection{Decomposition of tensor product representations}
\label{sec:tensor}

Proofs of Section~\ref{sec:fourier_descriptor_on_se}, will use a well-known fact on tensor product representations: the Induction-Reduction Theorem.
(See \citep{Barut77theoryof}).
This theorem allows to decompose the tensor products of representations $T^{\lambda_1}\otimes T^{\lambda_2}$, acting on $\bC^N\otimes \bC^N\cong \bC^{N\times N}$, to an equivalent representation acting on $\bigoplus_{k\in\bZ_N} \bC^N$, which is a block-diagonal operator whose block elements are of the form $T^{\lambda_1+R_k\lambda_2}$.
Moreover, the linear transformation realizing the equivalence is explicit.

To avoid confusion, we will henceforth denote components of vectors $v\in\bC^N$ as $v(0),\ldots, v(N-1)$, elements of $\bigoplus_{k\in\bZ_N} \bC^N$ as $(w_k)_{k\in\bZ_N}$ where $w_k\in\bC^N$, and the components of vectors $\textbf{v}\in\bC^N\otimes\bC^N$ as $\textbf{v}(k,h)$ for $k,h=0,\ldots, N-1$.
We also remark that linear operators $\cB$ on $\bigoplus_{k\in\bZ_N}\bC^N$ can be decomposed as $\cB = (\cB^{k,h})_{k,h\in\bZ_N}$, where each $\cB^{k,h}$ is an $N\times N$ complex matrix.
Namely, we have 
\begin{equation}
	\label{eq:ope}
	\cB(w_k)_{k\in\bZ_N} = \left( \sum_{h=0}^{N-1}\cB^{k,h}w_h \right)_{k\in\bZ_N}.
\end{equation}

Then, exploiting the commutation of the Fourier transform with equivalences of representation, the \\ Induction-Reduction Theorem implies that for every $\varphi\in L^2(SE(2,N))$ and any $\lambda_1,\lambda_2\in\cS$ it holds
\begin{equation}
	\label{eq:ind-red}
	A\circ \hat \varphi(T^{\lambda_1}\otimes T^{\lambda_2}) \circ A^{-1} = \bigoplus_{k\in\bZ_N}  \hat\varphi (T^{\lambda_1+R_k\lambda_2}).
\end{equation}
Here, $A:\bC^N\otimes \bC^N \to \bigoplus_{k\in\bZ_N} \bC^N$ is given by
\begin{equation}
	\label{eq:def-A}
	(A \textbf{v})_k(h) = (A_k \textbf{v})(h) = \textbf{v}(h,h-k),
	\quad
	\forall \textbf{v}\in\bC^N\otimes\bC^N.
\end{equation}


\section{Fourier descriptor on $SE(2,N)$} 
\label{sec:fourier_descriptor_on_se}

In the following sections we introduce and study the Fourier descriptors on the group $SE(2,N)$. As already mentioned, proving a general completeness result is essentially hopeless, and we will content ourselves to prove the weak completeness. 

Let $K\subset \bR^2$ be a compact set. In the following we will be mainly concerned with functions that are compactly supported either in $K$ or in $K\times\bZ_N\subset SE(2,N)$.

\subsection{Generalized Fourier descriptors}

Following \cite{Smach2008Generalized}, the power spectrum and the bispectrum on $\bR^2$ can be generalized to $SE(2,N)$ as follows.

\begin{definition}
	The  \emph{generalized power-spectrum} and \emph{bispectrum} of $\varphi\in L^2(SE(2,N))$ are the collections of matrices for any $\lambda,\lambda_1,\lambda_2\in\cS$,
	\begin{gather}
		\PS_\varphi(\lambda) := \widehat{\varphi}(T^\lambda) \circ \widehat{\varphi}(T^\lambda)^* \\
		\BS_\varphi(\lambda_1,\lambda_2) := \widehat{\varphi}(T^{\lambda_1})\otimes\widehat{\varphi}(T^{\lambda_2}) \circ \widehat{\varphi}(T^{\lambda_1}\otimes T^{\lambda_2})^*.
	\end{gather}
\end{definition}

%

The next result generalizes, with a simplified proof, the result presented in \citep{Smach2008Generalized}. 
Let us mention that this result is indeed true in a more general setting, as it will be shown in a forthcoming paper by Prandi and Gauthier.

\begin{theorem}
	\label{thm:complete-bisp}
	Let $K\subset \bR^2$ be a compact.
	The generalized power-spectrum and bispectrum are weakly complete on $L^2(\bZ_N\times K)$.
	In particular, they discriminate on the open and dense set $\cG\subset L^2(\bZ_N\times K)$ of functions $\varphi$ supported in $\bZ_N\times K$ and whose Fourier transform $\hat\varphi(T^\lambda)$ is invertible for an open and dense set of $\lambda$'s.
	That is, $\varphi_1,\varphi_2\in\cG$ are such that  $\PS_{\varphi_1}=\PS_{\varphi_2}$ and $\BS_{\varphi_1}=\BS_{\varphi_2}$ if and only if $\varphi_1 = \Lambda(x,k)\varphi_2$ for some $(x,k)\in SE(2,N)$.
\end{theorem}

\begin{proof}
	The fact that $\cG$ is open and dense is proved in Lemma~\ref{lem:open-dense} in Appendix~\ref{sec:proof-main}.
	Let $\varphi, \eta \in \cG$ be such that $PS_{\varphi_1}=PS_{\varphi_2}$ and $\BS_\varphi=\BS_\eta$.
	The equality of the generalized bispectrum implies that the set of $\lambda$'s where $\hat\varphi(T^\lambda)$ and $\hat\eta(T^\lambda)$ fail to be invertible is the same. 
	We will denote it by $I$ and let 
	\begin{equation}
		U(T^\lambda):=\hat\varphi(T^\lambda)^{-1} \hat\eta(T^\lambda) \qquad \forall \lambda\in I.
	\end{equation}
	In order to complete the proof of the statement, we will prove that $U(T^\lambda)$ can be defined for all $\lambda$'s in $\bR^2$ and, moreover, that $U(T^\lambda) = T^\lambda(a)$ for some $a\in SE(2,N)$. Indeed, by \eqref{eq:fundamental} this will readily implies that $\varphi = \Lambda(a)\eta$ as announced.

	We claim that $U(T^\lambda)$ is unitary for all $\lambda\in I$.	
	Indeed, by the equality of the generalized power-spectrum we have
	\begin{equation}
		U(T^\lambda)^* U(T^\lambda) = \eta(T^\lambda)^* \PS_\varphi(\lambda) \eta(T^\lambda) = \idty.
	\end{equation}

	Observe that the equality of the generalized bispectrum and the definition of $U$, imply that for all $\lambda_1,\lambda_2\in I$ it holds
	\begin{multline}
		\BS_\varphi(\lambda_1,\lambda_2) = \\
		\hat\varphi(T^{\lambda_1})\otimes\hat\varphi(T^{\lambda_2})\circ U(T^{\lambda_1})\otimes U(T^{\lambda_2}) \circ \hat\eta(T^{\lambda_1}\otimes T^{\lambda_2})^*.
	\end{multline}
	By the invertibility of $\hat\varphi(T^{\lambda_1})\otimes\hat\varphi(T^{\lambda_2})$ and the unitarity of $U$, this yields
	\begin{equation}\label{eq:tensor}
		\hat\varphi(T^{\lambda_1}\otimes T^{\lambda_2})\circ U(T^{\lambda_1})\otimes U(T^{\lambda_2}) = \hat\eta(T^{\lambda_1}\otimes T^{\lambda_2}).
	\end{equation}
	
	The announced result is then a consequence of the following three facts, which are proved in Appendix~\ref{sec:proof-main}.
	\begin{enumerate}
		\item Lemma~\ref{lem:cont}: The function $\lambda\mapsto U(T^\lambda)$ is continuous on $I$.
		\item Lemma~\ref{lem:extension}: The function $\lambda\mapsto U(T^\lambda)$ can be extended to a continuous function on $\bR^2$ for which \eqref{eq:tensor} is still true.
		\item Lemma~\ref{lem:final}: There exists $a\in SE(2,N)$ such that $U(T^\lambda) = T^\lambda(a)$. 	\qed
	\end{enumerate}
\end{proof}

An immediate corollary is the following.

\begin{corollary}\label{cor:cyclic}
	Let $\tilde \cL: L^2(\bR^2)\to L^2(SE(2,N))$ be an injective lift operator (not necessarily satisfying \eqref{eq:twist-shift}). Assume that there exists a residual set $\cR\subset L^2(\bR^2)$ such that $\tilde \cL(\cR)\cap \cG$ is residual. Then, the generalized power-spectrum and bispectrum are weakly complete on $L^2(\bR^2)$. Namely, for any $f,g\in \cR$ it holds that $\BS_{\tilde \cL f} = \BS_{\tilde \cL g}$ if and only if $f=\pi(x,k)g$ for some $(x,k)\in SE(2,N)$.
\end{corollary}

\begin{remark}
In \cite{Smach2008Generalized}, the authors applied their version of Theorem~\ref{thm:complete-bisp} to a non-left-invariant lift $\tilde \cL$, called \emph{cyclic lift}.
Indeed, for this cyclic lift, when $N$ is odd, it is possible to prove that, for any compact $K\subset \bR^2$, there exists a residual set $\cR\subset L^2(K)$ satisfying the assumptions of Corollary~\ref{cor:cyclic}.
\end{remark}

Unfortunately, Corollary~\ref{cor:cyclic} can never be applied to lifts of the form \eqref{eq:twist-shift}.
In fact, as proved in Appendix~\ref{sec:proofs}, letting $\omega_f(\lambda):=(\hat f(R_{-k}\lambda))_{k=0}^{N-1}\in\bC^N$, we have that
\begin{equation}
	\label{eq:ft-lift}
	\widehat{\cL f}(T^\lambda) = \omega_\Psi(\lambda)^* \otimes \omega_f(\lambda)^*,
\end{equation}
where for $v,w\in\bC^N$, we let $v^* = (\overline {v_k})_k$ and $(v\otimes w)_{k,h} = v_k\,\overline{w_h}$, so that $(v\otimes w)u = \langle w, u\rangle \, v$ for all $u\in \bC^N$.
This immediately implies that $\text{rank } \widehat{\cL f}(T^\lambda)\le 1$ and hence that $\text{range }\cL\cap \cG = \emptyset$ whenever $N>1$.

\subsection{Rotational Fourier descriptors}
\label{sec:rotational-bisp}

To bypass the difficulty posed by the non-invertibility of the Fourier transform for lifted functions, we are led to consider the following stronger descriptors.

\begin{definition}
	The  \emph{rotational power-spectrum} and \emph{bispectrum} of $\varphi\in L^2(SE(2,N))$ are the collections of matrices, for any $\lambda,\lambda_1,\lambda_2\in\cS$ and $h\in\bZ_N$,
	\begin{gather}
		\RPS	\varphi(\lambda,h) := \widehat{\varphi}(T^{R_h\lambda}) \circ \widehat{\varphi}(T^\lambda)^* \\
		\RBS_\varphi(\lambda_1,\lambda_2,h) := \widehat{\varphi}(T^{R_h{\lambda_1}})\otimes\widehat{\varphi}(T^{\lambda_2}) \circ \widehat{\varphi}(T^{\lambda_1}\otimes T^{\lambda_2})^*.
	\end{gather}
\end{definition}

As already mentioned in the introduction, the rotational descriptors are invariant only under the action of $\bZ_N\subset SE(2,N)$ but not under translations.
To avoid this problem, let us fix a compact $K\subset \bR^2$ and consider the set $\cA\subset L^2(\bR^2)$ of functions compactly supported in $K$, with non-zero average\footnote{Recall that the average of $f:\bR^2\to\bR$ is $\text{avg } f = \int_{\bR^2} f(x)\,dx$, which is always well-defined for $L^2(\bR^2)$ functions with compact support.}. Observe that this is an open and dense subset of $L^2(K)$.
We can then define the barycenter $c_f\in\bR^2$ of $f\in\cA$ as
\begin{equation}
c_f=\frac{1}{\text{avg } f} \left(\int_{\bR^2}x_1 f(x)\,dx, \int_{\bR^2}x_2 f(x)\,dx\right),
\end{equation}
and the centering operator $\Phi:\cA\to \cA$ as 
\begin{equation}\label{eq:centering}
	\Phi f(x) \coloneqq f(x-c_f).
\end{equation}
Then, considering the centered lift $\cL_c = \cL\circ\Phi$, we have that $\cL_c f = \cL_c g$ if and only if $g$ is a translate of $f$.
In particular, 
\begin{multline}
	\cL_c f = \Lambda(0,k)\cL_c g \\
	\iff
	f = \pi(x,k) g 
	\quad
	\text{for some }x\in\bR^2.
\end{multline}

Let us consider the following set of functions. 

\begin{definition}\label{def:set-image}
	Let $\cR\subset L^2(\bR^2)$ be the set of real-valued functions $f$ supported in $K$, such that $\hat f(\lambda)\neq 0$ for a.e.\ $\lambda\in\bR^2$ and the family $\Omega_f = \{S^k \omega_f(\lambda)\}_{k=0}^{N-1}$ is a basis for $\bC^N$, if $N$ is odd, or, if $N$ is even, for
	\begin{equation}
		\cX = \{ v\in \bC^N \mid v(h) = \overline{v(h+N/2)} \quad \forall h\in\bZ_N \}.
	\end{equation}
\end{definition}

The dependence of this definition on the parity of $N$ comes from the well-known fact that $\hat f(\lambda) = \overline{\hat f(-\lambda)}$. Indeed, for $N$ even, this implies that $S^k\omega_f(\lambda)\in \cX$ for any $k\in\bZ_N$. As such, there is no hope for the family $\Omega_f$ to generate the whole $\bC^N$.

Finally, we have the following Theorem. 

\begin{theorem}
	\label{thm:rot-bisp-compl}
	For any compact $K\subset \bR^2$, if the mother wavelet $\Psi\in \cR$, the rotational power-spectrum and bispectrum are weakly complete on $L^2(K)\cap \cA$.
   Namely, the set $\cR$ is open and dense in $L^2(K)$ and for any $f,g\in\cR\cap \cA$ it holds that $\RPS_{\cL_c f}=\RPS_{\cL_c g}$ and $\RBS_{\cL_c f}=\RBS_{\cL_c g}$ if and only if $f = \pi(x,k)g$ for some $(x,k)\in SE(2,N)$.
\end{theorem}

Here, we content ourselves to present only a sketch of the proof of this result for the case $N$ odd. The parity of $N$ does not introduce essential problems, up to exploit the fact that $\text{range } \widehat{Lf}(T^\lambda)\subset \cX$ for all $f\in\cR$ and $\lambda\in\cS$ and that the equivalence $A$ of the Induction-Reduction Theorem quotients nicely to an equivalence between $\cX\times\cX$ and $\bigoplus_{k\in\bZ_N}\cX$.
However, in order to prove the key technical point \eqref{eq:tensor-rot} we need a much finer study of the properties of circulant operators, which is outside the scope of this work and we defer to a forthcoming paper by Prandi and Gauthier.

\begin{proof}[Sketch in the case $N$ odd]
	The fact that $\cR$ is open and dense in $L^2(K)$ follows from the same arguments in Lemma~\ref{lem:open-dense}.
	
	Let $\Circ v$ be the circulant matrix associated with $v$, that is, $\Circ v = [v,Sv,\ldots,S^{N-1}v]$. Then the condition on $\Omega_f$ for $f\in\cR$ is equivalent to the invertibility of $\Circ\omega_f(\lambda)$ for an open and dense set of $\lambda$'s.
	By the properties of the Fourier transform on $\bR^2$ w.r.t.\ translations it follows that
	\begin{equation}
		\omega_{\Phi f}(\lambda) = \diag_k\left( e^{-i\langle\lambda, R_k c_f \rangle} \right) \omega_f(\lambda).
	\end{equation}
	This entails that $\Circ \omega_f(\lambda)$ is invertible if and only if $\Circ  \omega_{\Phi f}(\lambda)$ is.
	Hence, the statement is equivalent to the fact that for any couple $f,g\in \cR$ we have $\RBS_{\cL f}=\RBS_{\cL g}$ if and only if $f = R_k g$ for some $k\in\bZ_N$.

	The proof is similar to the one of Theorem~\ref{thm:complete-bisp}, but with additional technical difficulties.
	Let $I$ be the set where $\Circ \omega_f(\lambda)$ and $\Circ\omega_g(\lambda)$ are invertible. By assumption $I$ is open and dense.
	To overcome the non-invertibility of $\widehat{\cL f}$ in the definition the candidate intertwining representation $U$, we exploit the invertibility of the circulant matrices $\Circ \omega_f(\lambda)$ and $\Circ \omega_g(\lambda)$ on an open and dense set.
	Namely, for any $\lambda\in I$ we let
	\begin{equation}
		U(T^\lambda)^* := \Circ\omega_g(\lambda)\,\left(\Circ \omega_f(\lambda) \right)^{-1}.
	\end{equation}
    By definition, $U(T^\lambda)$ is circulant and $U(T^\lambda)^* S^k\omega_f(\lambda) = S^k \omega_g(\lambda)$ for any $k\in\bZ_N$.
	Moreover, by \eqref{eq:ft-lift}, this is equivalent to 
	\begin{equation}
		\widehat{\cL f}(T^{R_k\lambda}) \, U(T^\lambda) = \widehat{\cL g}(T^{R_k\lambda}), \qquad \forall k\in\bZ_N.
	\end{equation}
	In particular, $\lambda \mapsto U(T^{\lambda})$ is constant on orbits $\{R_k\lambda\}_{k\in\bZ_N}$.
    Finally, $U(T^\lambda)$ is unitary as a consequence, e.g., of Theorem~\ref{thm:explicit-comp}.
    
    The main difficulty in the proof is now to derive the equivalent of identity \eqref{eq:tensor}, that is, that for an open and dense set of couples $(\lambda_1,\lambda_2)$ we have
    \begin{multline}\label{eq:tensor-rot}
    	\widehat{\cL f}(T^{R_k\lambda_1}\otimes T^{R_k\lambda_2}) \, U(T^{\lambda_1})\otimes U(T^{\lambda_2}) \\
    	= \widehat{\cL g}(T^{R_k\lambda_1}\otimes T^{R_k\lambda_2}), \quad \forall k\in\bZ_N.
    \end{multline}
    As already mentioned, the proof of this identity requires a deep use of properties of circulant operators, which is outside the scope of this paper. We thus defer it to a forthcoming paper.
    
    Once \eqref{eq:tensor-rot} is known, the statement follows applying the same arguments as those in Theorem~\ref{thm:complete-bisp}. Namely,
    \begin{enumerate}
		\item The function $\lambda\mapsto U(T^\lambda)$ is continuous on $I$. This can be done via the same arguments as in Lemma~\ref{lem:cont}.
		\item The function $\lambda\mapsto U(T^\lambda)$ can be extended to a continuous function on $\cS$ still satisfying \eqref{eq:tensor-rot}. This can be done exactly as in Lemma~\ref{lem:extension}.
		\item There exists $k\in \bZ_N$ such that $U(T^\lambda) = T^\lambda(0,k)$. This is proved following Lemma~\ref{lem:final}. Indeed, the fact that now $\lambda\mapsto U(T^\lambda)$ is constant on the orbits $\left\{ R_k \lambda\right\}_{k\in\bZ_N}$ implies that the $\varphi_k$'s obtained there have to be independent of $k$. Since $\varphi_k(\lambda)=e^{i\langle R_k x_0, \lambda\rangle}$ for some $x_0\in\bR^2$, this implies that $x_0 =0$ and hence $\varphi_k\equiv 0$.
	      Obviously, this proves that $U(T^\lambda)=S^k = T^\lambda(0,k)$, for some $k\in\bZ_N$.\qed
	\end{enumerate}
\end{proof}

\subsection{Practical computation of the Fourier descriptors}

Here, we present some explicit formulae for the computation of the Fourier descriptors presented in this section.
%
%

In the following, we show that, under some assumptions on the mother wavelet $\Psi$, the concrete computation of the generalized power-spectrum and bispectrum and of their rotational counterparts, depend only on the 2D Fourier transform of $f$. 

%

\begin{theorem}
	\label{thm:explicit-comp}
	Assume that the mother wavelet $\Psi\in\cR$.
	Then:
	\begin{itemize}[leftmargin=*]
		\item For any $f\in \cR$, the generalized power-spectrum and bispectrum of $\cL f$ are respectively determined by the quantities, for a.e.\ $\lambda,\lambda_1,\lambda_2\in\cS$,
		\begin{gather*}
			I_1^{\lambda}(f) = \|\omega_f(\lambda)\|^2 = \sum_{k=0}^{N-1} |\hat f(R_{-k}\lambda)|^2 \\
			\begin{split}
			I_1^{\lambda_1,\lambda_2}(f) 
			&= \langle\omega_f(\lambda_1)\odot \omega_f(\lambda_2), \omega_f(\lambda_1+\lambda_2)\rangle \\
			&=\sum_{k=0}^{N-1} \hat f(R_{-k}\lambda_1)\hat f(R_{-k}\lambda_2) \overline{\hat f(R_{-k}(\lambda_1+\lambda_2))}.
			\end{split}
		\end{gather*}
		\item For any $f\in\cA\cap\cR$, the rotational power-spectrum and bispectrum of $\cL_c f$ are respectively determined by the quantities, for a.e.\ $\lambda,\lambda_1,\lambda_2\in\cS$ and $h\in\bZ_N$,
		\begin{gather}
			\label{eq:rot-power}
			\begin{split}
				I_2^{\lambda,h}(f) 
				&= \langle\omega_{\Phi f}(R_h\lambda), \omega_{\Phi f}(\lambda)\rangle \\
				&= 
				\sum_{k=0}^{N-1} \overline{\hat f(R_{-k+h}\lambda)} \hat f(R_{-k}\lambda), \\
			\end{split}\\
			\begin{split}
				&I_2^{\lambda_1,\lambda_2,h}(f)
				= \langle\omega_{\Phi f}(R_h\lambda_1)\odot \omega_{\Phi f}(\lambda_2), \omega_{\Phi f}(\lambda_1+\lambda_2)\rangle\\
				&=
				\sum_{k=0}^{N-1} \hat f(R_{-k+h}\lambda_1)\hat f(R_{-k}\lambda_2) \overline{\hat f(R_{-k}(\lambda_1+\lambda_2))}.
			\end{split}
		\end{gather}
		Here, $\Phi:\cA\to\cA$ is the centering operator defined in \eqref{eq:centering}.
	\end{itemize}
\end{theorem}

\begin{remark}\label{rmk:cyclic-comp}
	Theorem~\ref{thm:explicit-comp} shows in particular that the result of Theorem~\ref{thm:rot-bisp-compl} is indeed stronger than the completeness result for the generalized bispectrum of the cyclic lift obtained in \citep{Smach2008Generalized}.
	Indeed, in that work is proved that the latter (for odd $N$) is determined exactly by the quantities, for a.e.\ $\lambda_1,\lambda_2\in\cS$ and $h,k\in\bZ_N$,
	\begin{multline}
		\tilde I_2^{\lambda_1,\lambda_2, k, h} = \\
		\langle\omega_{\phi f}(R_h\lambda_1)\odot \omega_{\phi f}(R_k\lambda_2), \omega_{\phi f}(\lambda_1+R_{h+k}\lambda_2)\rangle.
	\end{multline}
	In particular, for each $\lambda_1,\lambda_2\in\cS$ one has to compute $N$ times more quantities than those for the rotational bispectrum.
\end{remark}

As a corollary of Theorem~\ref{thm:explicit-comp} we show that, in order to compare the power-spectra and bispectra, it is usually enough to compare only the latter.

\begin{corollary}
	Let $\Psi\in\cR$ and $f,g\in\cR\cap \cA$. Then, if $\cL f$ and $\cL g$ have the same generalized (resp.\ rotational) bispectrum, they have also the same generalized (resp.\ rotational) power-spectrum.
\end{corollary}

\begin{proof}
	We only prove the result for the rotational descriptors. In order to prove the one for the generalized descriptors, it will be enough to fix $h=0$ in the following.
	By Theorem~\ref{thm:explicit-comp} it is enough to show that whenever $I^{\lambda_1,\lambda_2,h}_2(f)= I^{\lambda_1,\lambda_2,h}_h(g)$ for a.e.\ $\lambda_1,\lambda_2\in\cS$ and any $h\in\bZ_N$, then $I^{\lambda,h}_2(f) = I^{\lambda,h}_1(g)$ for a.e. $\lambda\in\cS$ and any $h\in\bZ_N$.
	We start by observing that by the Paley-Wiener Theorem all these quantities are analytic, since $f$ and $g$ are compactly supported. Moreover,
	\begin{equation}
		\lim_{\lambda_1,\lambda_2\downarrow 0} I_2^{\lambda_1,\lambda_2,h}(f) 
		= N \, \hat f(0) |\hat f(0)|^2
		=  N \, \avg(f)^3,
	\end{equation}
	and the same is true for $g$. Thus, $\avg(f) = \avg(g)$.
	Finally, the result follows observing that
	\begin{equation}
		\lim_{\lambda_2\downarrow0} I_2^{\lambda_1,\lambda_2,h}(f) = \avg(f)\,I^{\lambda_1,h}_2(f). 
		\qed
	\end{equation}
\end{proof}

\section{Experimental results} 
\label{sec:experimental_results}

The goal of this section is to evaluate the performance of the invariant Fourier descriptors defined in the previous section on a large image database for object recognition. 
In addition to the generalized power-spectrum (PS) and bispectrum (BS) and the rotational power-spectrum (RPS) and bispectrum (RBS), we also consider the combination of the RPS and BS descriptors.
Indeed, combining these two descriptors seems to be a good compromise between the theoretical result of completeness given by Theorem~\ref{thm:rot-bisp-compl}, which only holds for the RBS, and computational demands, as the results on the COIL-100 database will show.

After showing how to efficiently compute these descriptors and presenting the image data set, we analyze some experimental results. 
In order to estimate the features capabilities, we use a support vector machine (SVM)  \cite{vapnik1998statistical}	 as supervised classification method. 
The recognition performances of the different descriptors regarding invariance to rotation, discrimination capability and robustness against noise are compared.

\subsection{Implementation}

As proved in Theorem~\ref{thm:explicit-comp}, the equality of the Fourier descriptors we introduced does not depend on the choice of the mother wavelet $\Psi$. Accordingly, in our implementation we only computed the quantities introduced in Theorem~\ref{thm:explicit-comp}, whose complexity is reduced to the efficient computation of the vector $\omega_f (\lambda )$, for a given $\lambda\in\cS$. 
We recall that this vector is obtained by evaluating the Fourier transform of $f$ on the orbit of $\lambda$ under the action of discrete rotations $R_{-k}$ for $k \in \bZ_N$.

Let us remark that, although in our implementation we chose this approach, in principle fixing a specific mother wavelet could be useful to appropriately weight descriptors depending on the associated frequencies.
Indeed, preliminary tests with a Gabor mother wavelet (which can be easily shown to be in $\cR$) showed slightly better results at a bigger computational cost.

%

For the implementation we chose to consider $N = 6$ and to work with images composed of hexagonal pixels. 
There are two reasons for this choice:
\begin{itemize}
	\item It is well-known that retinal cells are distributed in a hexagonal grid, and thus it is reasonable to assume that cortical activations reflect this fact.
	\item Hexagonal grids are invariant under the action of $\bZ_6$ and discretized translations, which is the most we can get in the line of the invariance w.r.t.\ $SE(2,6)$. Indeed, apart from the hexagonal lattice, the only other lattices on $\bR^2$ which are invariant by some $\bZ_N$ and appropriate discrete translations are obtained with $N=2,3,4$.
\end{itemize}
The different steps of computation of the descriptors\footnote{MATLAB sample code for the implementation of the rotational bispectral invariants can be found at \url{https://nbviewer.jupyter.org/github/dprn/bispectral-invariant-svm/blob/master/Invariant_computation_matlab.ipynb}} are described in Figure~\ref{fig:schema} and given as follow:
\begin{enumerate}
	\item The input image is converted to grayscale mode, the Fourier transform is computed via FFT, and the zero-frequency component is shifted to the center of the spectrum (Fig~\ref{fig:schema}. S1).
	\item For cost computational reasons and since we are dealing with natural images, for which the relevant frequencies are the low ones, we extract a grid of $16\times16$ pixels around the origin (Fig \ref{fig:schema}. S2).
	\item The invariants of Theorem~\ref{thm:explicit-comp} are computed from the shifted Fourier transform values, on all frequencies in an hexagonal grid inside this $16\times 16$ pixels square.
	A bilinear interpolation is applied to obtain the correct values of $\omega _f (\lambda )$ (Fig \ref{fig:schema}. S3, S4, S5, S6).
	The final dimension of the feature-vector is given in Table~\ref{tab:feat-vec}. 
\end{enumerate}

\begin{table}
	\centering
	\caption{Dimension of the feature vectors for the Fourier descriptors under consideration}
	\label{tab:feat-vec}
	\begin{tabular}{|c|c|}
	\hline
	Descr. &  Dim. \\
	\hline
	\hline 
	PS  	&   136  	\\ \hline
	BS  	& 	717 	\\ \hline
	RPS  	& 	816 	\\ \hline
	RBS  	& 	4417 	\\ \hline
	RPS + BS  	& 	1533 	\\ \hline
	\end{tabular}
\end{table}

\begin{figure*}[ht]
\centering
  \includegraphics[width=1\textwidth]{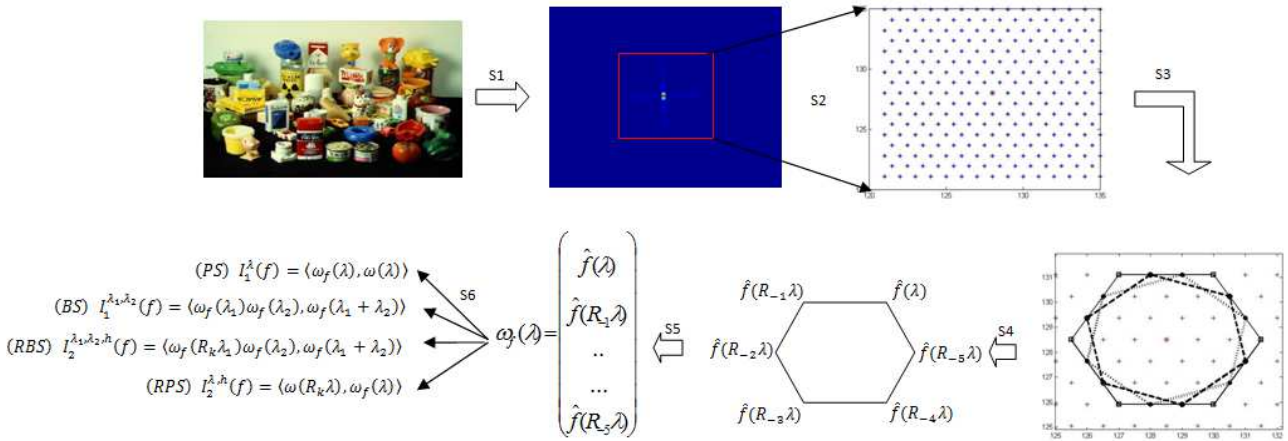}
\caption{Steps of computing the invariant descriptors. (S1) computation of the shifted $FFT$ of the the image $f$, (S2) generation of the hexagonal grid, (S3) extraction of different hexagons, (S4) evaluation of the $FFT$ of $f$ on each extracted hexagon, (S5) generation of the vector $\omega _f (\lambda )$ and (S6) computation of the four invariants }
\label{fig:schema}       
\end{figure*}

\subsection{Test protocol}

We use the Fourier descriptors to feed an SVM classifier, via the MATLAB Statistics and Machine Learning Toolbox, applying it on a database of 7200 objects extracted from the Columbia Object Image Library (COIL-100)
and a database of 400 faces extracted from ORL
face database.
Finally, we compare the results obtained with those obtained using traditional descriptors.

The result of the training step consists of the set of support vectors determined by the SVM based method.
During the decision step, the classifier computes the Fourier descriptors and the model determined during the training step is used to perform the SVM decision. 
The output is the image class.

For COIL-100 database, two cases are studied: a case without noise and another with noise. In the first one, tests have been performed using 75\% of the COIL-100 database images for training and 25\% for testing.In the second one, we have used a learning data-set composed of all the 7200 images (100 objects with 72 views) without noise and a testing data-set composed of 15 randomly selected views per object to which an additive Gaussian noise with $S_d$ of 5, 10 and 20 was added. (See Fig. \ref{fig:noise}). 

We evaluate separately the recognition rate obtained using the four previous invariant descriptors and the combination of the RPS \& BS invariants to test their complementarity.
Then, we compare their performance with the Hu's moments (HM), the Zernike's moments (ZM), the Fourier-Mellin transform (FM), described in Appendix~\ref{sec:moment}, and the local SIFT and HOG descriptors \citep{dalal2005histograms} whose performance under the same conditions has been tested in \citep{Choksuriwong2008}, 

Since we use the RBF kernel in the SVM classification process, this depends on the kernel size $\sigma$. 
The results presented here are obtained by choosing empirically the value $\sigma _{opt}$ that provided maximum recognition rate.

\subsection{Experiments}
The performances of the different invariant descriptors are analyzed with respect to the recognition rate given a learning set. Hence, for a given ratio, the learning and testing sets have been built by splitting randomly all examples. Then, due to randomness of this procedure, multiple trials have been performed with different random draws of the learning and testing set. 
In the case of an added noise, since as mentioned before the learning set is comprised of all images, this procedure is applied only to the testing set.

The parameters of our experiments are the following:

\begin{enumerate}
\item The learning set $c_i$ corresponding to the values of an invariant descriptor computed on an image from the database;
\item The classes $\hat c_i  \in \left\{ {1,100} \right\}$ corresponding to the object class.
\item Algorithm performance: the efficiency is given \\through a percentage of the well recognized objects composing the testing set.
\item Number of random trials: fixed to 5.
\item Kernel K: a Gaussian kernel of bandwidth $\sigma$ is chosen
\begin{equation}
K(x,y) = e^{\frac{{ - \left\| {x - y} \right\|^2 }}{{2\sigma ^2 }}} \\ 
\end{equation}
$x$ and $y$ correspond to the descriptors vectors of objects.
\end{enumerate}

For solving a multi-class problem, the two most popular approaches are the one-against-all (OAA) method and the one-against-one (OAO) method \citep{milgram:inria-00103955}. 
For our purpose, we chose an OAO SVM because it is substantially faster to train and seems preferable for problems with a very large number of classes. 

\subsubsection{COIL-100 databases}
The Columbia Object Image Library (COIL-100, Fig. \ref{fig:coil}) is a database of color images of 100 different objects, where 72 images of each object were taken at pose intervals of $5^\circ$.

\subsubsection*{\textbf{Classification performance}}
         
\begin{figure}[ht]
\centering
  \includegraphics[width=0.48\textwidth]{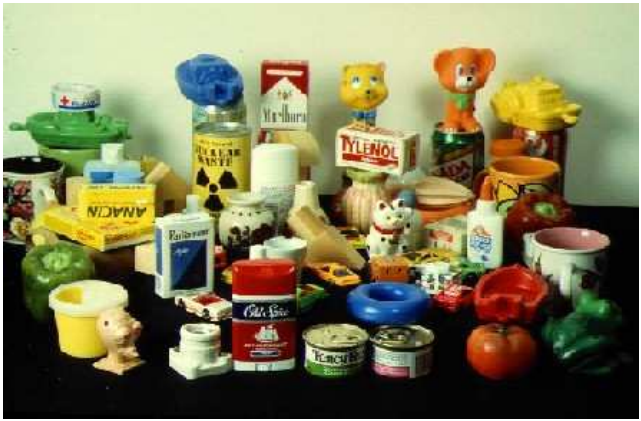}
\caption{Sample objects of COIL-100 database}
\label{fig:coil}       
\end{figure}

Table~\ref{tab:coilTab} presents results obtained testing our object recognition method with the COIL-100 database.
The best results were achieved using the local SIFT descriptor. The RBS comes in the second place and the local HOG features come third. Indeed it has been demonstrated in the literature, these local methods currently give the best results. However, if noise is added on the image, the use of global approach is better than the use of local ones. The main reason is that the key-points detector used in the local method produce in these cases many key-points that are nor relevant for object recognition. This will be shown in the next subsection.

\begin{table}[ht]
\centering
\caption{Recognition rate for each descriptor using the COIL-100 database. The test results for ZM, HM, FM, and SIFT are taken from \cite{Choksuriwong2008}.}
\label{tab:coilTab}       
\begin{tabular}{|c|c|}
\hline
Descriptors & Recognition rates \\
\hline
\hline
RBS & \textbf{95.5\%}\\
BS & 88\%\\
PS & 84.3\%\\
RPS &89.8\%\\
RPS+BS & \textbf{92.8\%}\\
ZM & 91.9\%\\
HM & 80.2\%\\
FM & 89.6\% \\
HOG & \textbf{95.3\%} \\
SIFT & \textbf{100\%} \\\hline
\end{tabular}
\end{table}

\begin{figure}[ht]
\centering
  \includegraphics[width=0.5\textwidth]{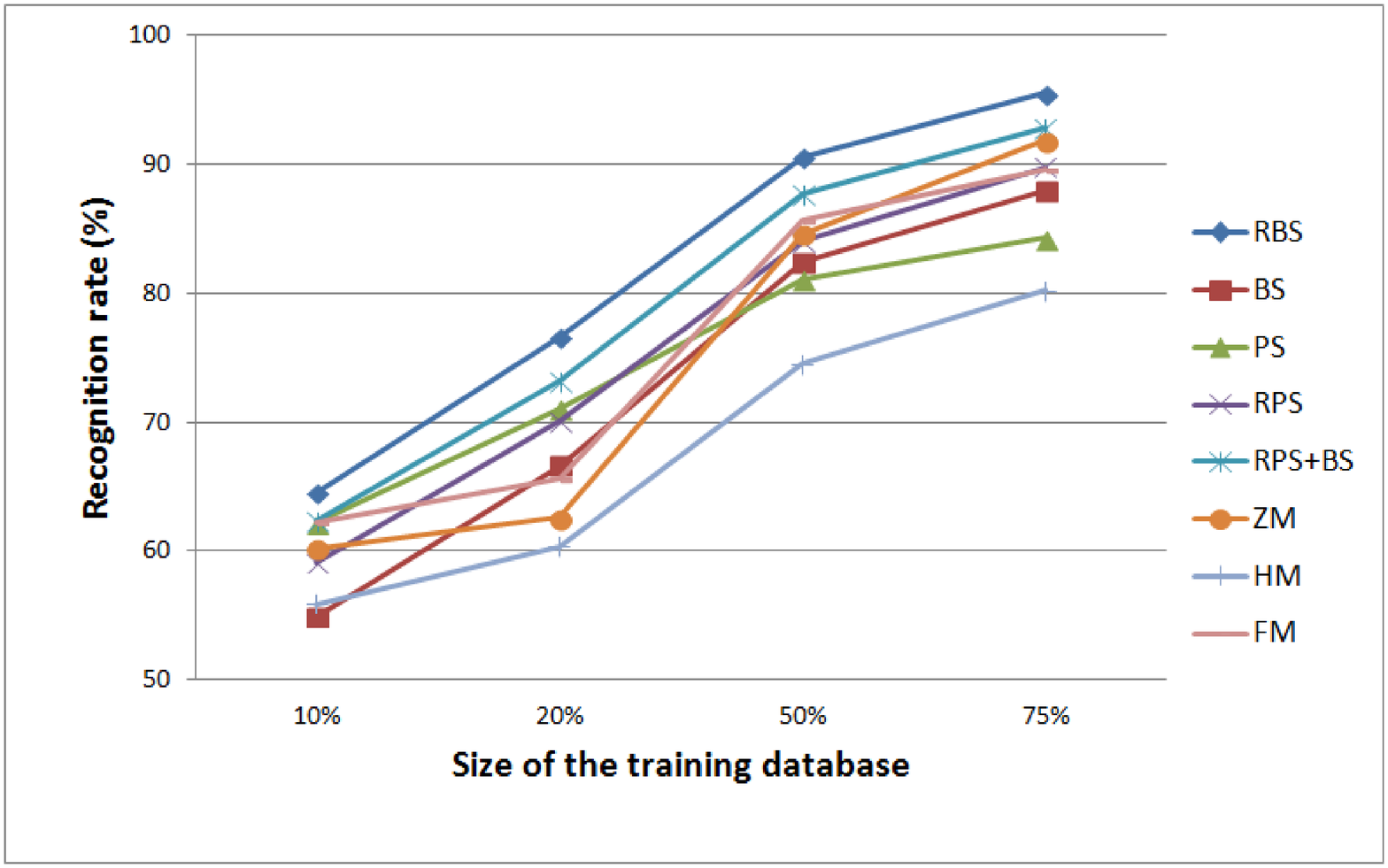}
\caption{Classification rate for different size of the training database. The test results for ZM, HM, FM, and SIFT are taken from \cite{Choksuriwong2008}.}
\label{fig:graph}       
\end{figure}

In Figure \ref{fig:graph} we present the recognition rate as a function of the size of the training set. 
As expect, this is an increasing function and we remark that the RBS and the combination of the RPS and the BS give better results than the other global invariant descriptors.

\subsubsection*{\textbf{Robustness against noise}}

Also in this case, test results for ZM, HM, FM, and SIFT are taken from \cite{Choksuriwong2008}.

\begin{figure}[h]
\centering
  \includegraphics[width=0.48\textwidth]{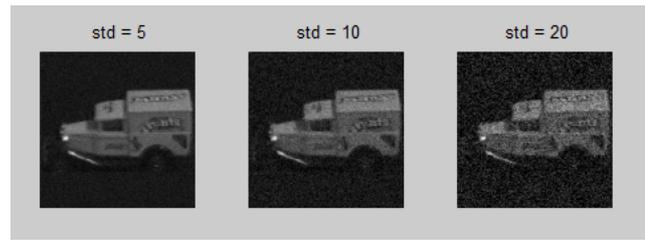}
\caption{Sample of COIL-100 noisy object}
\label{fig:noise}       
\end{figure}

Results presented in Table~\ref{tab:feat-noise} show that noise has little influence on classification performance when we use a global descriptor such as RBS, BS, the combination of BS \& RPS, ZM, HM and FM. It has however a sensible effect on the SIFT local descriptor, and a big one on the HOG local descriptor. 

\begin{table*}[t]
	\centering
	\caption{Classification rate on COIL-100 noisy database. The test results for ZM, HM, FM, and SIFT are taken from \cite{Choksuriwong2008}.}
	\label{tab:feat-noise}
	\begin{tabular}{|c|c|c|c|c|c|c|c|c|c|c|}
	\hline
	$S_d$ & RBS & BS & PS & RPS & RPS+BS & ZM & HM & FM & SIFT & HOG\\
	\hline
	\hline 
	5 & \textbf{100\%} & \textbf{100\%} & 71.5\% & 99.8\% & \textbf{100\%} & \textbf{100\%} & 95.2\% & 98.6\% & 89.27\% & 4\%  	\\ 
	10 & \textbf{100\%} & \textbf{100\%} & 71.2\% & 99.8\% & \textbf{100\%} & \textbf{100\%} & 95.2\% & 95.2\% & 88.89\% & 1.2\%  	\\ 
	20 & \textbf{100\%} & \textbf{100\%} & 67.8\% & 99.8\% & \textbf{100\%} & \textbf{100\%} & 91.4\% & 90.2\% & 85.46\% & 1\%  	\\ \hline

	\end{tabular}
\end{table*}

\subsubsection{The ORL database}

The Cambridge University ORL face database (Fig. \ref{fig:orl}) is composed of 400 grey level images of ten different patterns for each of 40 persons. 
The variations of the images are across time, size, pose and facial expression (open/closed eyes, smiling/not smiling), and facial details (glasses/no glasses). 

\begin{figure}[ht]
\centering
  \includegraphics[width=0.48\textwidth]{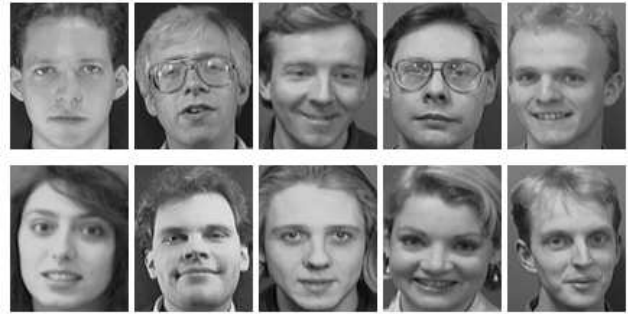}
\caption{Face samples from the ORL database}
\label{fig:orl}       
\end{figure}

In the literature, the protocol used for training and testing is different from one paper to another. In \\\citep{341300}, a hidden Markov model (HMM) based approach is used, and the best model resulted in recognition rate of 95\%, with high computational cost. In \citep{hjelmas2001}, Hjelmas reached a 85\% recognition rate using the ORL database and feature vector consisting of Gabor coefficients.

We perform experiments on the ORL database using the RBS, BS, PS, RPS, ZM, HU, FM, and the combination of the RPS \& BS descriptors. 
Since the local descriptors SIFT and HOG obtained, predictably, almost perfect scores, we do not present them.
The results are shown in Table \ref{tab:orlTab}, where we clearly see that the RBS invariant descriptor gives the best recognition rate $c = 89.8\%$, faring far better than before w.r.t.\ the combination of RPS and BS descriptors.

\begin{table}[ht]
\centering
\caption{Recognition rate for each descriptor using the ORL database}
\label{tab:orlTab}       
\begin{tabular}{|c|c|}
\hline
Descriptors & Recognition rates\\
\hline
\hline
RBS & \textbf{89.8\%}\\
BS & 67.9\%\\
PS & 49.2\%\\
RPS & 76.9\%\\
RPS+BS & 79.8\%\\
ZM & 75\%\\
HM & 43.5\%\\
FM & 47.6\%\\
\hline
\end{tabular}
\end{table}


\section{Conclusion and perspectives}
\label{sec:conclusion}

In this paper we presented four Fourier descriptors over the semidiscrete roto-translation group $SE(2,N)$. Then, we proved that the generalized power-spectrum ($\PS$) and bispectrum ($\BS$) -- and thus the rotational power-spectrum ($\RPS$) and bispectrum ($\RBS$) -- are weakly complete, in the sense that they allow to distinguish over an open and dense set of compactly supported functions $\varphi\in L^2(SE(2,N))$ up to the $SE(2,N)$ action. This generalizes a result of \cite{Smach2008Generalized}. 
We then considered a framework for the application of these Fourier descriptors to roto-translation invariant object recognition, inspired by some neurophysiological facts on the human primary visual cortex. 
In this framework, we showed that the rotational bispectrum is indeed a weakly complete roto-translation invariant for planar images.
Moreover, although the proposed Fourier descriptors are given in terms of complex mathematical objects, we showed that they can be implemented in a straightforward way as linear combinations of the values of the 2D Fourier transform of the image.

In the second part of the paper, we proposed an evaluation of the performances of these Fourier descriptors  in object recognition and we presented the results obtained on different databases: the COIL-100 database, composed of several objects undergoing 3D rotation and scales changes, and the ORL-database, on which different human faces are subject to several kind of variations. 
For both these databases, the global Fourier descriptors introduced in this paper are the most efficient global descriptors tested, equalled only, for noisy images, by the Zernike Moments.
Although for unperturbed images the local SIFT descriptor gives better recognition rate, the addition of noise  leads to the global descriptors outperforming the local ones.
These results thus show the rotational bispectrum ($\RBS$) to be a very good Fourier descriptor for object recognition, consistently with the theoretical weak completeness result. When the dimension of the feature vector is an issue, the $\RBS$ can be replaced by a combination of the generalized bispectrum ($\BS$) and the rotational power-spectrum ($\RPS$), which yields slightly worse results with a feature vector of length almost one third.

An extension of the object recognition method presented in this paper to an AdaBoost framework for the problem of object detection is currently undergoing. 

\appendix

\section{Moment invariants and Fourier-Mellin transform}
\label{sec:moment}

In this section, following \cite{Choksuriwong2008}, we review the two most used classes of moment invariants, Hu and Zernike, and Fourier-Mellin descriptors, that we use as a comparison for our generalized Fourier descriptors.

Moment invariants were first introduced to the pattern recognition and image processing community in 1962 by Hu \citep{Hu1962}, with the introduction of the seven Hu moments which are invariants under translation, rotation and scaling. These are derived from a scaling and translation invariant modification of the standard moments of an image $I:\mathbb R^2\to \mathbb R$. Namely, 
\begin{equation}
	v_{p,q}  = \frac{{u_{p,q} }}{{u_{0,0}^{(1 + \frac{{p + q}}{2})} }},
\end{equation}
where
\begin{equation}
	u_{p,q}  = \int\limits_{\bR^2 } {(x-x_0)^p (y-y_0)^q I (x,y)dxdy},
\end{equation}
and $x_0  = \frac{{m_{1,0} }}{{m_{0,0} }}$ and $y_0  = \frac{{m_{0,1} }}{{m_{0,0} }}$ are the coordinates of the barycenter computed via the standard $(p+q)$-th order moments of $I$:
\begin{equation}
	m_{p,q}  = \int_{\bR^2 } {x^p y^q I(x,y)dxdy}.
\end{equation}

Another important class of moments are the Zernike ones, introduced in \citep{chong2003translation} and computed via orthogonal Zernike polynomials. The Zernike moment of order $(m,n)$ is:
\begin{equation}
Z_{mn}  = \frac{{m + 1}}{n}\sum\limits_x {\sum\limits_y {I(x,y)\left[ {V_{mn} (x,y)} \right]} } 
\end{equation}
where $x^2  + y^2  < 1$ and $V_{mn} (x,y)$ are the Zernike polynomials defined in polar coordinates as $V_{mn} (r,\theta ) = R_{mn} (r)e^{jn\theta }$, where
\begin{equation}
R_{mn} (r) = \sum\limits_{s = 0}^{\frac{{m - \left| n \right|}}{2}} {\frac{{( - 1)^s (m - s)!r^{m - 2s} }}{{s!(\frac{{m + \left| n \right|}}{2} - s)!(\frac{{m - \left| n \right|}}{2} - s)!}}} 
\end{equation}
These moments present several advantages. Indeed, beside a rotation and translation invariance they have nice orthogonality properties and are considered to be robust against image noise. In particular, the orthogonality property helps in achieving a near zero value of redundancy measure in a set of moments functions \citep{Chong2003}. 

Finally, strictly related to Fourier descriptors are the descriptors obtained via Fourier-Mellin transform (FMT), presented in \citep{derrode2001robust}. The FMT of an image $I$, that we assume to be given in polar coordinates, is defined as:
\begin{equation}
M_I (u,v) = \frac{1}{{2\pi }} \int_0^{2\pi}\int_0^\infty {I(r,\theta )r^{ - iv} e^{ - iu\theta } \frac{{dr}}{r}d\theta } .
\end{equation} 
Following \citep{derrode2001robust}, we will indeed compute the analytical Fourier-Mellin transform (AFMT). That is, we replace $I$ in the above definition with its regularized version $I_\sigma  (r,\theta ) = r^\sigma  I(r,\theta )$, where $\sigma>0$. Finally, each feature $M_{I_\sigma } (u, v)$ is modified in order to compensate for the rotation, translation and size changes of the object.

\section{Auxiliary lemmata for the proof of Theorem~\ref{thm:complete-bisp}}
\label{sec:proof-main}

\begin{lemma}
	\label{lem:open-dense}
	The set  $\cG$ introduced in Theorem~\ref{thm:complete-bisp} is open and dense in $L^2(K\times\bZ_N)$.
\end{lemma}

\begin{proof}
	We start by showing that $\cG\neq\varnothing$. To this aim, it suffices to consider $\varphi$ such that $\varphi(\cdot,k)\equiv0$ for all $k\in \bZ_N\setminus\{0\}$ and $\varphi(\cdot,0)\neq 0$ such that $\supp\,\cF(\varphi(\cdot,0))=\bR^2$. By \eqref{eq:non-comm-ft}, we then have $\varphi\in\cG$, since
	\begin{equation}
		\det\hat\varphi(T^\lambda) = \prod_{k\in\bZ_N}\cF(\varphi(\cdot,0))(R_{-k}\lambda)\neq 0 \qquad\forall\lambda\in\cS.
	\end{equation}

	For any $\varphi\in \cG$ and $k\in \bZ_N$, the Paley-Wiener Theorem implies that $\cF(\varphi(\cdot, k))$ is analytic. In particular, by \eqref{eq:non-comm-ft}, $\lambda\mapsto \det \hat\varphi(T^\lambda)$ is analytic. Thus, $\varphi \in \cG$ if and only if $\varphi(T^{\lambda_0})$ is invertible for some $\lambda_0\in\cS$.
	
	We claim that the set $\cG$ is dense. Indeed, let $\varphi\notin\cG$ and fix some $\eta\in \cG$ and $\lambda_0\in\cS$ such that $\hat\eta(T^{\lambda_0})$ is invertible. By analyticity of $\varepsilon \mapsto \det(\hat\varphi(T^{\lambda_0})+\varepsilon\hat\eta(T^{\lambda_0}))$ follows that $\varphi+\varepsilon \eta\in\cG$ for sufficiently small $\varepsilon>0$, which entails that $\varphi\in \bar\cG$, proving the claim.
	
	Let us prove that $\cG$ is open in $L^2(K\times\bZ_N)$. To this aim, fix $\varphi\in\cG$ and $\varphi_n\rightarrow\varphi$ in $L^2(K\times\bZ_N)$. This implies that $\hat\varphi_n\rightarrow\hat\varphi$ in $L^2(\widehat{SE(2,N)})$, and in particular that $\hat\varphi_n\rightarrow\hat\varphi$ in measure. By definition of convergence in measure, this implies that for sufficiently big $n$ it has to hold $\det\hat\varphi_n(T^{\lambda_0})\neq0$. Hence $\varphi_n\in\cG$ for $n$ sufficiently big and $\cG$ is open.
	\qed
\end{proof}

Before diving into the proofs of the other auxiliary lemmata, we make the following observation.
Let $\lambda_1,\lambda_2\in I$ be such that $\lambda_1+R_k\lambda_2\in I$ for all $k\in \bZ_N$.
Applying the Induction-Reduction theorem \eqref{eq:ind-red} to \eqref{eq:tensor} yields
\begin{multline}\label{eq:ind-red-cont-U-moore}
		A \circ U(T^{\lambda_1})\otimes U(T^{\lambda_2}) \circ A^{-1} \\
		= \bigoplus_{k\in\bZ_N} \varphi(T^{\lambda_1+R_k\lambda_2})^{-1} \eta(T^{\lambda_1+R_k\lambda_2})
		= \bigoplus_{k\in\bZ_N} U(T^{\lambda_1+R_k\lambda_2}).
\end{multline}

\begin{lemma}
	\label{lem:cont}
	The function $\lambda\mapsto U(T^\lambda)$ is continuous on $I$.
\end{lemma}

\begin{proof}	
	Fix $\lambda_0\in I$ and an open set $V\subset \bR^2$ such that 
	\begin{equation}
		\int_V U(T^{\lambda_2})^*_{i,j}\,d\lambda_2 >0 \qquad\forall i,j\in\bZ_N.
	\end{equation}
	This is possible since $U\not\equiv 0$. Since the set $I$ is open dense, up to reducing $V$ we can assume that there exists a neighborhood $W$ of $\lambda_0$ such that $V+\lambda\subset I$ for any $\lambda\in W$. 
	Then, \eqref{eq:ind-red-cont-U-moore} holds for $\lambda_1\in W$ and $\lambda_2\in V$.
	Explicitly computing the $0,0$ block of \eqref{eq:ind-red-cont-U-moore}, we have
	\begin{equation}\label{eq:cont}
		U(T^{\lambda_1})_{i,j} U(T^{\lambda_2})_{i,j}^* = U(T^{\lambda_1+\lambda_2})_{i,j} \qquad \forall i,j \in\bZ_N.
	\end{equation}
	Then, integrating it over $V$ w.r.t.\ $\lambda_2$ yields
	\begin{equation}
		U(T^{\lambda_1})_{i,j} = \frac{\int_{V+\lambda} U(T^{\lambda_2})_{i,j}\,d\lambda_2}{\int_{V} U(T^{\lambda_2})^*_{i,j}\,d\lambda_2} 
		\qquad\forall \lambda_1\in W, \, \forall i,j\in\bZ_N
	\end{equation}
	Since the function on the r.h.s.\ is clearly continuous on $W$ this proves the continuity at $\lambda_0$ of $\lambda\mapsto U(T^\lambda)$, completing the proof.
	\qed
\end{proof}

\begin{lemma}
	\label{lem:extension}
	The function $\lambda\mapsto U(T^\lambda)$ can be extended to a continuous function on $\bR^2$ for which \eqref{eq:tensor} is still true.
\end{lemma}

\begin{proof}
	Let $\lambda_0\notin I$.
  Since $I$ is an open and dense set, this implies that $\lambda_0$ is in its closure and that we can choose $\lambda_1,\lambda_2\in I$ such that $\lambda_0=\lambda_1+R_{k_0} \lambda_2$ for some $k_0\in\bZ_N$ and $\lambda_1+R_k\lambda_2\in I$ for any $k\neq k_0$.
  We then let
  \begin{equation}
    \label{eq:ind-red-Moore2}
    U(T^{\lambda_0}) \coloneqq \left(A \circ U(T^{\lambda_1})\otimes U(T^{\lambda_2})\circ A^*\right)^{k_0,k_0}.
  \end{equation}

  We now prove that the above definition does not depend on the choice of $\lambda_1$, $\lambda_2$ and $k_0$.
  By openness of $I$, there exists a neighborhood $V$ of $\lambda_2$ entirely contained in $I$.
  Then, up to taking a smaller $V$, it holds that $\lambda_1+R_{k_0}\lambda_2'\in I$ for any $\lambda_2'\in V\setminus\{\lambda_2\}$.
  By \eqref{eq:ind-red-cont-U-moore}, this implies that for any $\mu_1+R_{\ell}\mu_2=\lambda_0$ it holds 
  \begin{multline}
  (A \circ U(T^{\lambda_1})\otimes U(T^{\lambda_2'})\circ A^*)^{k_0,k_0}=\\
  (A \circ U(T^{\mu_1})\otimes U(T^{\mu_2'})\circ A^*)^{\ell,\ell}.	
  \end{multline}
   for $\lambda_2'$ and $\mu_2'$ sufficiently near, but different, to $\lambda_2$ and $\mu_2$, respectively.
  By the continuity of $U$ on $I$, proved in Lemma~\ref{lem:cont}, this implies that this equation has to hold also for $\lambda_2'=\lambda_2$  and $\mu_2'=\mu_2$.
  Hence, \eqref{eq:ind-red-Moore2} does not depend on the choice of $\lambda_1,\lambda_2$ and $k_0$.

  Finally, the fact that $\hat f(T^{\lambda_1}\otimes T^{\lambda_2}) \circ U(T^{\lambda_1})\otimes U( T^{\lambda_2}) = \hat g(T^{\lambda_1}\otimes T^{\lambda_2})$ for any $\lambda_1,\lambda_2$ follows from \eqref{eq:ind-red-Moore2} and \eqref{eq:ind-red-cont-U-moore}.
  \qed
\end{proof}

\begin{lemma}
	\label{lem:final}
	There exists $a\in SE(2,N)$ such that $U(T^\lambda) = T^\lambda(a)$.
\end{lemma}

\begin{proof}
	By definition of $U$ it holds that
  \begin{equation}
  \bigoplus_{k\in\bZ_N} U(T^{\lambda_1+R_k\lambda_2}) \circ A = A\circ U(T^{\lambda_1})\otimes U(T^{\lambda_2}) \quad\forall \lambda_1,\lambda_2\neq0.
  \end{equation}
  Then, for any $i,j,\ell,k$,
  \begin{equation}
  \label{eq:final-comp-moore}
  U(T^{\lambda_1})_{\ell,i}U(T^{\lambda_2})_{\ell-k,j} = 
      \begin{cases}
      U(T^{\lambda_1+R_k\lambda_2})_{\ell,i} &\quad \text{if }j=i-k,\\
      0 &\quad \text{otherwise.}
      \end{cases}
  \end{equation}

  By invertibility of $U(T^{\lambda_1})$, there exists $i_{0}\in\bZ_N$ such that $U(T^{\lambda_1})_{0,i_0}\neq 0$.
  Using \eqref{eq:final-comp-moore} this implies that $U(T^{\lambda_2})_{-k,j} =0$ for any $j\neq i_{0}-k$.
  Namely, we have proved that there exists a family of functions $\varphi_{-k}:\cS\to \bC$ such that $U(T^{\lambda_1})_{-k,\cdot} = \varphi_{-k}(\lambda_1) \,\delta_{i_{0}-k}$ or, equivalently, that
  \begin{equation}
  	U(T^\lambda) = \text{diag}_k \varphi_k(\lambda)\,S^{i_0}.
  \end{equation} 
  By the explicit expression \eqref{eq:repre} of $T^\lambda$, in order to complete the proof it suffices to prove that $\varphi_k(\lambda)= e^{i\langle x_0, R_{-k} \lambda\rangle}$ for some $x_0\in\bR^2$.

  By continuity and unitarity of $U$, the $\varphi_k$'s are continuous and satisfy $|\varphi_h(\lambda)|=1$.
  Using again \eqref{eq:final-comp-moore} with $j=i_0-k$, we obtain 
  \begin{equation}
    \label{eq:character-moore}
    \varphi_{\ell}(\lambda_1+R_k\lambda_2) = \varphi_\ell(\lambda_1)\varphi_{\ell-k}(\lambda_2),
  \end{equation}
  for any $\lambda_1,\lambda_2\neq0$  and $\ell,k\in\bZ_N$.

	We claim that the $\varphi_\ell$'s are characters of $\bR^2$.
  Indeed, let us fix $k=0$ in \eqref{eq:character-moore}:
  \begin{equation}
  	\label{eq:char}
  	\varphi_{\ell}(\lambda_1+\lambda_2) = \varphi_\ell(\lambda_1)\varphi_{\ell}(\lambda_2).
  \end{equation}
  Choosing $\lambda_2=-\lambda_1$ in the above shows that $\varphi_\ell$ can be extended at $0$.
  Moreover, letting $\lambda_1=0$ and taking the limit $\lambda_2\rightarrow 0$ shows that this extension is continuous.
  Since characters of $\bR^2$ are exactly the continuous functions satisfying \eqref{eq:char}, the claim is proved.
  
  By Pontryiagin duality, there exists $x_\ell\in\bR^2$ such that $\varphi_\ell(\lambda) = e^{i\langle\lambda,x_\ell\rangle}$.
  Finally, by \eqref{eq:character-moore} with $k\in\bZ_N$ one obtains that $R_{-k}x_\ell=x_{\ell-k}$, which proves that there exists $x_0\in\bR^2$ such that $\varphi_\ell(\lambda)=e^{i\langle x_0, R_{-k} \lambda\rangle}$.
  This completes the proof of the statement.
  \qed
\end{proof}

\section{Proofs} 
\label{sec:proofs}

\begin{proof}[Formula \eqref{eq:ft-lift}]
	Let $\lambda\in\cS$ and consider $v\in \bC^N$.
	Observe that $(x,k)^{-1} = (-R_{-k}x,-k)$.
	Then, by \eqref{eq:ft-def}, \eqref{eq:twist-shift}, and \eqref{eq:repre}, for any $h\in\bZ_N$ we have
	\begin{equation}
		\begin{split}
			&(\widehat{\cL}(T^\lambda).v)_h \\
			& = \sum_{k=0}^{N-1}\int_{\bR^2} \cL(x,k) e^{-i\langle \lambda, R_{h-k} x\rangle} v_{h-k} \, dx \\
			& = \sum_{k=0}^{N-1} v_{h-k}\, \int_{\bR^2} \int_{\bR^2} f(y) \bar\Psi(R_{-k}(y-x)) e^{-i\langle \lambda, R_{h-k} x\rangle}  \,dy\, dx \\
			& = \sum_{k=0}^{N-1} v_{h-k}\, \int_{\bR^2}\int_{\bR^2} \bar\Psi(z) f(y)  e^{-i\langle R_{k-h}\lambda, y-z\rangle}  \,dy\, dz \\
			& = \bar{\hat\Psi}(R_h\lambda) \sum_{k=0}^{N-1} v_{h-k}\, \hat f(R_{k-h}\lambda)\\
			& = \overline{\omega_\Psi(\lambda)}_h \langle \overline{\omega_f(\lambda)} , v \rangle. 
		\end{split}
	\end{equation}
	By definition of $\omega_\Psi(\lambda)^*\otimes \omega_f(\lambda)^*$ this completes the proof.
	\qed
\end{proof}

In order to prove Theorem~\ref{thm:explicit-comp} we need the following explicit description of the equivalence in the Induction-Reduction theorem of Section~\ref{sec:tensor}.

\begin{lemma}
	\label{lem:tensor-ind}
	For any $M, N\in\bC^{N\times N}$ we have 
	\begin{equation}
		\left(A\circ (M\otimes N) \circ A^{-1}\right)^{k,h} = (M_{i,j}\, N_{i-k,j-h})_{i,j\in\bZ_N}.
	\end{equation}
\end{lemma}

\begin{proof}
	Observe that for any $\textbf{v}\in\bC^{N}\otimes\bC^N$ it holds $(M\otimes N).\textbf{v} = M\circ \textbf{v}\circ N^T$.
	Thus,
	\begin{equation}
		[A\circ(M\otimes N).\textbf{v}]_k(i) 
		= \sum_{j,\ell=0}^{N-1} M_{i,j} \, N_{i-k,\ell}\, \textbf{v}(j,\ell).
	\end{equation}
	Since is straightforward to check that $A^{-1}:\bigoplus_{k\in\bZ_N}\bC^N\to \bC^{N\times N}$ is given by $[A^{-1}(w_\ell)_{\ell\in\bZ_N}](k,h) = w_{k-h}(k)$, we then have
	\begin{equation}
		\begin{split}
			[A\circ(M\otimes N) \circ A^{-1}.(w_h & )_{h\in\bZ_N}]_k(i) \\
			&= \sum_{j,\ell=0}^{N-1} M_{i,j}\, N_{i-k,\ell}\, w_{j-\ell}(j)\\
			&= \sum_{j,h=0}^{N-1} M_{i,j}\, N_{i-k,j-h}\, w_{h}(j).			
		\end{split}
	\end{equation}
	By \eqref{eq:ope}, this completes the proof.
	\qed
\end{proof}

\begin{proof}[Proof of Theorem~\ref{thm:explicit-comp}]
	Without loss of generality we can restrict ourselves to consider functions such that $\phi f = f$ and $\cL_c f = \cL f$.
	We start by the trivial remark that the result on the rotational descriptors contains the one on the generalized ones.
	Let us consider
	\begin{gather}
		I_2^{\lambda,h}(f) 
		:= \langle\omega_{f}(R_h\lambda), \omega_{f}(\lambda)\rangle, \\
		I_2^{\lambda_1,\lambda_2,h}(f)
		:= \langle\omega_{f}(R_h\lambda_1)\odot \omega_{f}(\lambda_2), \omega_{f}(\lambda_1+\lambda_2)\rangle.
	\end{gather}
	Since $f$ is assumed to be compactly supported, its Fourier transform is analytic, and so are the functions $\lambda\mapsto I_2^{\lambda,h}(f)$ and $(\lambda_1,\lambda_2)\mapsto I_2^{\lambda_1,\lambda_2,h}(f)$ for any $h\in \bZ_N$.
	Thus, the statement of the proposition reduces to show that $\RPS_{\cL f}= \RPS_{\cL g}$ (resp. $RB_{\cL f} = RB_{\cL g}$) if and only if $I_2^{\lambda,h}(f) = I_2^{\lambda,h}(g)$ (resp. $I_2^{\lambda_1,\lambda_2,h}(f) = I_2^{\lambda_1,\lambda_2,h}(g)$) for a.e.\ $\lambda,\lambda_1,\lambda_2\in\cS$ and all $h\in\bZ_N$
	
	Let us recall the following properties of the tensor product, valid for all $v,v_1,v_2,w,w_1,w_2\in\bC^N$:
	\begin{enumerate}
		\item $(v\otimes w)^* = v\otimes w$,
		\item $(v_1\otimes w_1)\circ (v_2\otimes w_2) = \langle w_1,v_2 \rangle \, v_1\otimes w_2$,
	\end{enumerate}
	By these and \eqref{eq:ft-lift}, we immediately have
	\begin{equation}
		\begin{split}
		\RPS_{\cL_c f}(\lambda,h) 
		&= \langle \omega_{\phi f}(R_h\lambda)^*, \omega_{\phi f}(\lambda)^* \rangle \, \omega_\Psi(R_h\lambda)^*\otimes \omega_\Psi(\lambda)^*\\
		&= I_2^{\lambda,h}(f) \, \omega_\Psi(R_h\lambda)^*\otimes \omega_\Psi(\lambda)^*.
		\end{split}
	\end{equation}
	Hence, whenever $\omega_\Psi(R_h\lambda)^*\otimes \omega_\Psi(\lambda)^* \neq0$, $\RPS_{\cL_cf}(\lambda,h) = \RPS_{\cL_cg}(\lambda,h)$ if and only if  $I_2^{\lambda,h}(f)=I_2^{\lambda,h}(g)$.
	Since $\omega_\Psi(R_h\lambda)^*\otimes \omega_\Psi(\lambda)^* \neq 0$ if and only if $\omega_\Psi(\lambda) \neq 0$, by the fact that $\Psi\in\cR$ this is true for a.e.\ $\lambda\in\cS$.
	This completes the proof of the part of the statement regarding the rotational power-spectrum.

	To prove the statement regarding the rotational bispectrum, let $\cB_f = A\circ \RBS_{\cL f}(\lambda_1,\lambda_2,h)\circ A^{-1}$, where $A$ is the equivalence given by the Induction-Reduction Theorem and defined in \eqref{eq:def-A}.
	Since $A$ is invertible, determining \\$\RBS_{\cL f}(\lambda_1,\lambda_2,h)$ is equivalent to determining $\cB_f$.
	Exploiting the fact that the r.h.s.\ of \eqref{eq:ind-red} is a diagonal matrix we have
	\begin{multline}
			\cB_f^{k,\ell} 
			= \left(A\circ \widehat{\cL f}(T^{R_h\lambda_1})\otimes \widehat{\cL f}(T^{\lambda_1}) \circ A^{-1}\right)^{k,\ell} \\
			\circ \widehat{\cL f}(T^{\lambda_1+R_\ell\lambda_2})^*.
	\end{multline}
	By Lemma~\ref{lem:tensor-ind}, formula \eqref{eq:ft-lift}, and explicit computations, we then get
	\begin{multline}
		\cB^{k,\ell}_f = \left\langle \omega_f(R_h\lambda_1)\odot \omega_f(R_\ell\lambda_2), \omega_f(\lambda_1+R_\ell \lambda_2) \right\rangle \\
		\times \left(\omega_{\Psi}(R_h\lambda_1)^*\odot \omega_{\Psi}(R_{h+k}\lambda_2)^*\right)\otimes \omega_{\Psi}(\lambda_1+R_{\ell}\lambda_2).
	\end{multline}
	Similarly to before, $\left(\omega_{\Psi}(R_h\lambda_1)^*\odot \omega_{\Psi}(R_{h+k}\lambda_2)^*\right)\otimes \omega_{\Psi}(\lambda_1+R_{\ell}\lambda_2) \neq 0$ for a.e.\ $\lambda_1,\lambda_2\in\cS$ since $\Psi\in\cR$.
	For these couples, $\cB_f = \cB_g$ if and only if 
	\begin{multline}
		\left\langle \omega_f(R_h\lambda_1)\odot \omega_f(R_\ell\lambda_2), \omega_f(\lambda_1+R_\ell \lambda_2) \right\rangle\\
		= \left\langle \omega_g(R_h\lambda_1)\odot \omega_g(R_\ell\lambda_2), \omega_g(\lambda_1+R_\ell \lambda_2) \right\rangle.
	\end{multline}
	Finally, making the change of variables $R_\ell \lambda_2\mapsto \lambda_2$ completes the proof of the theorem.
	\qed
\end{proof}

\acknowledgement{
	This research has  been supported by the European Research Council, ERC StG 2009 ``GeCoMethods'', contract n. 239748. The second and last authors were partially supported by the Grant ANR-15-CE40-0018 of the ANR.
}



\begin{thebibliography}{40}
\providecommand{\natexlab}[1]{#1}
\providecommand{\url}[1]{{#1}}
\providecommand{\urlprefix}{URL }
\expandafter\ifx\csname urlstyle\endcsname\relax
  \providecommand{\doi}[1]{DOI~\discretionary{}{}{}#1}\else
  \providecommand{\doi}{DOI~\discretionary{}{}{}\begingroup
  \urlstyle{rm}\Url}\fi
\providecommand{\eprint}[2][]{\url{#2}}

\bibitem[{Barut and Raçzka(1977)}]{Barut77theoryof}
Barut A, Raçzka R (1977) Theory of group representations and applications

\bibitem[{Bay et~al(2006)Bay, Tuytelaars, and Van~Gool}]{bay2006surf}
Bay H, Tuytelaars T, Van~Gool L (2006) Surf: Speeded up robust features. In:
  Computer vision--ECCV 2006, Springer, pp 404--417

\bibitem[{Boscain et~al(2012)Boscain, Duplaix, Gauthier, and
  Rossi}]{Boscain2012Anthropomorphic}
Boscain U, Duplaix J, Gauthier JP, Rossi F (2012) Anthropomorphic image
  reconstruction via hypoelliptic diffusion. SIAM Journal on Control and
  Optimization pp 1--25

\bibitem[{Boscain et~al(2014{\natexlab{a}})Boscain, Chertovskih, Gauthier, and
  Remizov}]{Boscain2014Hypoelliptic}
Boscain U, Chertovskih R, Gauthier JP, Remizov A (2014{\natexlab{a}})
  Hypoelliptic diffusion and human vision: a semi-discrete new twist on the
  petitot theory. SIAM J Imaging Sci 7(2):669--695

\bibitem[{Boscain et~al(2014{\natexlab{b}})Boscain, Gauthier, Prandi, and
  Remizov}]{Boscain2014Image}
Boscain U, Gauthier JP, Prandi D, Remizov A (2014{\natexlab{b}}) {Image
  reconstruction via non-isotropic diffusion in Dubins / Reed - Shepp-like
  control systems.} In: 53rd IEEE Conference on Decision and Control, pp
  4278--4283

\bibitem[{Bressloff et~al(2001)Bressloff, Cowan, Golubitsky, Thomas, and
  Wiener}]{Bressloff2001Geometric}
Bressloff P, Cowan J, Golubitsky M, Thomas P, Wiener M (2001) Geometric visual
  hallucinations, euclidean symmetry and the functional architecture of striate
  cortex. Philosophical transactions of the Royal Society of London Series B,
  Biological sciences 356(April 2000):299--330

\bibitem[{Choksuriwong et~al(2008)Choksuriwong, Emile, Rosenberger, and
  Laurent}]{Choksuriwong2008}
Choksuriwong A, Emile B, Rosenberger C, Laurent H (2008) Comparative study of
  global invariant descriptors for object recognition. Journal of Electronic
  Imaging pp 1--35

\bibitem[{Chong et~al(2003{\natexlab{a}})Chong, Raveendran, and
  Mukundan}]{Chong2003}
Chong C, Raveendran P, Mukundan R (2003{\natexlab{a}}) {A comparative analysis
  of algorithms for fast computation of Zernike moments}. Pattern Recognition
  36(3):731--742

\bibitem[{Chong et~al(2003{\natexlab{b}})Chong, Raveendran, and
  Mukundan}]{chong2003translation}
Chong CW, Raveendran P, Mukundan R (2003{\natexlab{b}}) Translation invariants
  of zernike moments. Pattern recognition 36(8):1765--1773

\bibitem[{Chou et~al(1994)Chou, O'Neill, and Cheng}]{399871}
Chou J, O'Neill W, Cheng H (1994) Pavement distress classification using neural
  networks. In: Systems, Man, and Cybernetics, 1994. Humans, Information and
  Technology., 1994 IEEE International Conference on, vol~1, pp 397--401 vol.1,
  \doi{10.1109/ICSMC.1994.399871}

\bibitem[{Citti and Sarti(2006)}]{Citti2006Cortical}
Citti G, Sarti A (2006) A cortical based model of perceptual completion in the
  roto-translation space. Journal of Mathematical Imaging and Vision
  24(3):307--326

\bibitem[{Dalal and Triggs(2005)}]{dalal2005histograms}
Dalal N, Triggs B (2005) Histograms of oriented gradients for human detection.
  In: Computer Vision and Pattern Recognition, 2005. CVPR 2005. IEEE Computer
  Society Conference on, IEEE, vol~1, pp 886--893

\bibitem[{Derrode and Ghorbel(2001)}]{derrode2001robust}
Derrode S, Ghorbel F (2001) Robust and efficient fourier-mellin transform
  approximations for invariant grey-level image description and reconstruction.
  Computer Vision and Image Understanding 83(1):57--78

\bibitem[{Dubnov et~al(1997)Dubnov, Tishby, and Cohen}]{Dubnov1997}
Dubnov S, Tishby N, Cohen D (1997) Polyspectra as measures of sound texture and
  timbre. Journal of New Music Research 26(4):277--314

\bibitem[{Duits and Franken(2010{\natexlab{a}})}]{Duits2010LeftInvarianta}
Duits R, Franken E (2010{\natexlab{a}}) {Left-invariant parabolic Evolutions on
  SE(2) and Contour Enhancement via Invertible Orientation Scores. Part I:
  linear left-invariant diffusion equations on SE(2)}. Quarterly of Appl Math,
  AMS

\bibitem[{Duits and Franken(2010{\natexlab{b}})}]{Duits2010LeftInvariant}
Duits R, Franken E (2010{\natexlab{b}}) {Left-invariant parabolic evolutions on
  SE(2) and contour enhancement via invertible orientation scores. Part II:
  nonlinear left-invariant diffusions on invertible orientation scores}. Q Appl
  Math (0):1--38

\bibitem[{F\"{u}hr and Mayer(2002)}]{Fuhr2002a}
F\"{u}hr H, Mayer M (2002) {Continuous wavelet transforms from semidirect
  products : Cyclic representations and Plancherel measure}. J Fourier Anal
  Appl pp 1--23,
  \urlprefix\url{http://www.springerlink.com/index/G7TC4AANGTUC4HXW.pdf},
  \eprint{0102002v1}

\bibitem[{Galerne et~al(2011)Galerne, Gousseau, and Morel}]{Galerne2011Random}
Galerne B, Gousseau Y, Morel J (2011) Random phase textures: Theory and
  synthesis. IEEE Transactions on Image Processing 20(1):257--267

\bibitem[{Granlund(1972)}]{Granlund1972}
Granlund GH (1972) {Fourier Preprocessing for Hand Print Character
  Recognition}. IEEE Trans Comput C-21(2):195--201

\bibitem[{Hewitt and Ross(1963)}]{Hewitt1963Abstract}
Hewitt E, Ross K (1963) Abstract harmonic analysis - Volume 1. Springer-Verlag

\bibitem[{Hjelmas and Low(2001)}]{hjelmas2001}
Hjelmas E, Low B (2001) {Face Detection: A Survey}. Computer Vision and Image
  Understanding 83(3):236--274

\bibitem[{Hu(1962)}]{Hu1962}
Hu M (1962) {Visual pattern recognition by moment invariants}. Information
  Theory, IRE Transactions on 8(2):179--187

\bibitem[{Hubel and Wiesel(1959)}]{Hubel1959Receptive}
Hubel D, Wiesel T (1959) Receptive fields of single neurones in the cat's
  striate cortex. The Journal of physiology 148:574--591

\bibitem[{Kakarala(2012)}]{Kakarala2012Bispectrum}
Kakarala R (2012) The bispectrum as a source of phase-sensitive invariants for
  fourier descriptors: A group-theoretic approach. Journal of Mathematical
  Imaging and Vision 44(3):341--353

\bibitem[{Ke and Sukthankar(2004)}]{ke2004pca}
Ke Y, Sukthankar R (2004) Pca-sift: A more distinctive representation for local
  image descriptors. In: Computer Vision and Pattern Recognition, 2004. CVPR
  2004. Proceedings of the 2004 IEEE Computer Society Conference on, IEEE,
  vol~2, pp II--506

\bibitem[{Kuhl and Giardina(1982)}]{Kuhl1982}
Kuhl FP, Giardina CR (1982) {Elliptic Fourier features of a closed contour}.
  Comput Graph Image Process 18(3):236--258

\bibitem[{Lowe(2004)}]{lowe2004distinctive}
Lowe D (2004) Distinctive image features from scale-invariant keypoints.
  International journal of computer vision 60(2):91--110

\bibitem[{Mikolajczyk and Schmid(2005)}]{mikolajczyk2005performance}
Mikolajczyk K, Schmid C (2005) A performance evaluation of local descriptors.
  Pattern Analysis and Machine Intelligence, IEEE Transactions on
  27(10):1615--1630

\bibitem[{Milgram et~al(2006)Milgram, Cheriet, and
  Sabourin}]{milgram:inria-00103955}
Milgram J, Cheriet M, Sabourin R (2006) {"One Against One" or "One Against
  All": Which One is Better for Handwriting Recognition with SVMs?} In: Lorette
  G (ed) {Tenth International Workshop on Frontiers in Handwriting
  Recognition}, {Universit{\'e} de Rennes 1}, {Suvisoft}, La Baule (France),
  http://www.suvisoft.com

\bibitem[{Morel and Yu(2009)}]{morel2009asift}
Morel JM, Yu G (2009) Asift: A new framework for fully affine invariant image
  comparison. SIAM Journal on Imaging Sciences 2(2):438--469

\bibitem[{Nene et~al(1996)Nene, Nayar, Murase et~al}]{nene1996columbia}
Nene SA, Nayar SK, Murase H, et~al (1996) Columbia object image library
  (coil-20). Tech. rep., technical report CUCS-005-96

\bibitem[{Petitot(2008)}]{Petitot2008Neurogeometrie}
Petitot J (2008) Neurogéométrie de la vision - Modèles mathématiques et
  physiques des architectures fonctionnelles. Les Éditions de l'École
  Polytechnique

\bibitem[{Prandi et~al(2015)Prandi, Boscain, and Gauthier}]{Prandi2015}
Prandi D, Boscain U, Gauthier JP (2015) Image processing in the semidiscrete
  group of rototranslations. In: To appear in the Proceedings of the 2nd
  conference on Geometric Science of Information

\bibitem[{Raja and Shanmugam(2011)}]{raja2011artificial}
Raja DMS, Shanmugam A (2011) Artificial neural networks based war scene
  classification using invariant moments and glcm features: A comparative
  study. International Journal of Engineering Science and Technology 3(2)

\bibitem[{Rajasekaran and Pai(2000)}]{rajasekaran2000image}
Rajasekaran S, Pai GV (2000) Image recognition using simplified fuzzy artmap
  augmented with a moment based feature extractor. International Journal of
  Pattern Recognition and Artificial Intelligence 14(08):1081--1095

\bibitem[{Samaria and Harter(1994)}]{341300}
Samaria F, Harter A (1994) Parameterisation of a stochastic model for human
  face identification. In: Applications of Computer Vision, 1994., Proceedings
  of the Second IEEE Workshop on, pp 138--142

\bibitem[{Sifre and Mallat(2013)}]{Mallat}
Sifre L, Mallat S (2013) Rotation, scaling and deformation invariant scattering
  for texture discrimination. In: Computer Vision and Pattern Recognition
  (CVPR), 2013 IEEE Conference on, pp 1233--1240, \doi{10.1109/CVPR.2013.163}

\bibitem[{Smach et~al(2008)Smach, Lemaître, Gauthier, Miteran, and
  Atri}]{Smach2008Generalized}
Smach F, Lemaître C, Gauthier JP, Miteran J, Atri M (2008) Generalized
  {Fourier} descriptors with applications to objects recognition in {SVM}
  context. Journal of mathematical imaging and Vision 30:43--71

\bibitem[{Vapnik and Vapnik(1998)}]{vapnik1998statistical}
Vapnik VN, Vapnik V (1998) Statistical learning theory, vol~1. Wiley New York

\bibitem[{Zahn and Roskies(1972)}]{Zahn1972Fourier}
Zahn CT, Roskies RZ (1972) Fourier descriptors for plane closed curves. IEEE
  Transactions on Computers C-21(3)

\end{thebibliography}

%
%
\bibliographystyle{spbasic}

\end{document}